\def\year{2021}\relax
\documentclass[letterpaper]{article} %
\usepackage{aaai21}  %
\usepackage{times}  %
\usepackage{helvet} %
\usepackage{courier}  %
\usepackage[hyphens]{url}  %
\usepackage{graphicx} %
\usepackage{natbib}  %
\usepackage{caption} %
\usepackage{amsmath}
\usepackage{amssymb}
\usepackage{amsthm}

\usepackage{xcolor}

\newcommand{\R}{\mathbb{R}} %
\newcommand{\N}{\mathbb{N}} %
\newcommand{\defsym}{:=} %

\newcommand{\bb}[1]{\boldsymbol{#1}}
\newcommand{\Ident}{\mathcal{I}} %
\DeclareMathOperator{\determinant}{det}

\newcommand{\E}{\mathbb{E}} %
\newcommand{\distr}{\sim} %
\newcommand{\GP}{\mathcal{G}} %
\newcommand{\Norm}{\mathcal{N}} %
\newcommand{\Pp}{\mathbb{P}}
\newcommand{\pbl}{\left[}
\newcommand{\pbr}{\right]}

\newtheorem{defn}{Definition}
\newtheorem{proposition}[defn]{Proposition}

\newtheorem{theorem}[defn]{Theorem}

\newtheorem*{proposition*}{Proposition}

\title{Practical and Rigorous Uncertainty Bounds for Gaussian Process Regression}
\author{
    Christian Fiedler,\textsuperscript{\rm 1, \rm 2, \rm3}
    Carsten W. Scherer,\textsuperscript{\rm2}
    Sebastian Trimpe \textsuperscript{\rm 1, \rm3}\\
}
\affiliations{
    \textsuperscript{\rm 1}Intelligent Control Systems Group, Max Planck Institute for Intelligent Systems\\
    \textsuperscript{\rm 2}Mathematical Systems Theory, University of Stuttgart\\
    \textsuperscript{\rm 3}Institute for Data Science in Mechanical Engineering, RWTH Aachen University \\
    fiedler@dsme.rwth-aachen.de, 
    carsten.scherer@imng.uni-stuttgart.de, 
    trimpe@dsme.rwth-aachen.de
}

\begin{document}

\maketitle

\begin{abstract}
Gaussian Process Regression is a popular nonparametric regression method based on Bayesian principles that provides uncertainty estimates for its predictions. However, these estimates are of a Bayesian nature, whereas for some important applications, like learning-based control with safety guarantees, frequentist uncertainty bounds are required. Although such rigorous bounds are available for Gaussian Processes, they are too conservative to be useful in applications. This often leads practitioners to replacing these bounds by heuristics, thus breaking all theoretical guarantees. To address this problem, we introduce new uncertainty bounds that are rigorous, yet practically useful at the same time. In particular, the bounds can be explicitly evaluated and are much less conservative than state of the art results. Furthermore, we show that certain model misspecifications lead to only graceful degradation. We demonstrate these advantages and the usefulness of our results for learning-based control with numerical examples.
\end{abstract}

\section{Introduction}
Gaussian Processes Regression (GPR) is an established and successful nonparametric regression method based on Bayesian principles
\cite{rasmussen_williams_gp} which has recently become popular in learning-based control 
\cite{liu_control_gp_tutorial, kocijan_modelling_control_gp}. In this context, safety and performance guarantees are important
aspects \cite{astrom_murray,skogestad_postlethwaite}. 
In fact, the lack of rigorous guarantees has been identified as one of the major obstacles preventing the usage of
learning-based control methodologies in safety-critical areas like autonomous driving, human-robot interaction or medical devices,
see e.g. \cite{berkenkamp_safe_exploration_rl}. 
One approach to tackle this challenge is to use the posterior variance of GPR to derive frequentist
uncertainty bounds and combine these with robust control methods that can deal with the remaining uncertainty. 
This strategy has been successfully applied in a
number of works that also provide control-theoretical guarantees, 
cf. Section \ref{section.bayesian_frequ}.

These approaches rely on rigorous frequentist uncertainty bounds for GPR. Although such results are available
\cite[Theorem~6]{srinivas2010}, \cite[Theorem~2]{cg17}, they turn out to be very conservative and difficult
to evaluate numerically in practice and are, hence, replaced by heuristics. That is, instead of bounds obtained from theory,
much smaller ones are assumed, sometimes without any practical justification, cf. Section \ref{section.related_work} for more discussion. 
Unfortunately, using heuristic approximations leads to a breakdown of the theoretical guarantees of these control approaches,
as already observed for example in \cite{ledereretal_uniform_error_bounds}.
Furthermore, when deriving theoretical guarantees based on frequentist results like
\cite[Theorem~6]{srinivas2010} or \cite[Theorem~2]{cg17}, model misspecifications (like wrong hyperparameters of the underlying
GPR model or approximations such as the usage of sparse GPs) are ignored. Since model misspecifications
are to be expected in any realistic setting, the validity of such theoretical guarantees based on idealized assumptions is unclear.

In summary, rigorous and practical frequentist uncertainty bounds for GPR are currently not available. By practical we mean that
concrete, not excessively conservative numerical bounds can be computed based on reasonable and established assumptions 
and that these are robust against model misspecifications at least to some extent. 
We note that such bounds are of independent interest, for example, for Bayesian Optimization \cite{shahriarietal_review_bo}. 
In this work, we improve previous frequentist uncertainty bounds for GPR leading to practical, yet theoretically
sound results. In particular, our bounds can be directly used in algorithms and are sharp enough to avoid
the use of unjustified heuristics. Furthermore, we provide robustness results that can handle moderate model misspecifications.
Numerical experiments support our theoretical findings and illustrate the practical use of the bounds.

\section{Background}
\subsection{Gaussian Process Regression and Reproducing Kernel Hilbert Spaces} \label{section.gpr}
We briefly recall the basics of GPR, for more details we refer to \cite{rasmussen_williams_gp}.
A Gaussian Process (GP) over an (input or index) set $D$ is a collection of random variables,
such that any finite subset is normally distributed. A GP $f$ is uniquely determined by its mean function $m(x)=\E[f(x)]$
and covariance function $k(x,x^\prime)=\E[(f(x)-m(x))(f(x^\prime)-m(x^\prime))]$ and we write $f \distr \GP_D(m,k)$.
Without loss of generality we focus on the case $m \equiv 0$. Common covariance functions include the Squared Exponential (SE)
and Matern kernel.
Consider a Gaussian Process prior $f \distr \GP_D(0,k)$ and noisy data $(x_i,y_i)_{i=1,\ldots,N}$, where
$x_i \in D$ and $y_i = f(x_i) + \epsilon_i$ with i.i.d. $\Norm(0,\sigma^2)$ noise, then the posterior is also a GP,
with the \emph{posterior mean} $\mu_N$, \emph{posterior covariance} $k_N$ and \emph{posterior variance} given by
\begin{align*}
 \mu_N(x) & = \mu(x) + \bb{k}_N(x)^T(\bb{K}_N + \sigma^2 \Ident_N)^{-1}\bb{y}_N \\
 k_N(x,x^\prime) & = k(x,x^\prime) - \bb{k}_N(x)^T(\bb{K}_N + \sigma^2 \Ident_N)^{-1}\bb{k}_N(x^\prime) \\
 \sigma_N^2(x) & = k_N(x,x),
\end{align*}
where we defined the kernel matrix $\bb{K}_N=(k(x_j,x_i))_{i,j=1,\ldots,N}$ and the column vectors
$\bb{k}_N(x)=(k(x_i,x))_{i=1,\ldots,N}$ and $\bb{y}_N = (y_i)_{i=1,\ldots,N}$.

Later on we follow \cite{srinivas2010,cg17} and assume that the ground truth is a function from a
Reproducing Kernel Hilbert Space (RKHS). 
For an introduction to the RKHS framework we refer to \cite[Chapter~4]{svm_book} or \cite{berlinet_kernel}
as well as \cite{kanagawaetal_gp_kernels} for connections between Gaussian Processes and the RKHS framework.

\subsection{Bayesian and Frequentist Bounds} \label{section.bayesian_frequ}
GPR is based on Bayesian principles: A prior distribution is chosen (here a GP with given
mean and covariance function) and then updated with the available data using Bayes rule, assuming a certain likelihood or noise
model (here independent Gaussian noise). The updated distribution is called the posterior distribution and can be interpreted as
a tradeoff between prior belief (encoded in the prior distribution together with the likelihood model) 
and evidence (the data) \cite[Chapter~5]{murphy_ml}. 
In contrast to the Bayesian approach, in frequentist statistics the existence of a ground truth is assumed and noisy
data about this ground truth is acquired \cite[Chapter~6]{murphy_ml}.

For many applications which require safety guarantees it is important to get reliable frequentist uncertainty bounds.
A concrete setting and major motivation for this work is learning-based robust control.
Here, the ground truth is an only partially known dynamical system and the goal is to solve a certain control task,
like stabilization or tracking, by finding a suitable controller. 
A machine learning method like GPR is used to learn more about the unknown dynamical system from data.
Thereafter a set of possible models is derived from the learning method that contains the ground truth with a given (high) probability.
For this reason such a set is often called an uncertainty set.
We then use a robust method on this set, i.e., a method the leads to a controller that works 
for every element of this set. Since the ground truth is contained in this set with a given (high) probability,
the task is solved with this (high) probability.
Examples of such an approach are \cite{kolleretal_learning_mpc_exploration} (using robust model predictive control
with state constraints), \cite{umlauftetal_fb_lin_gp} (using feedback linearization to achieve ultimate boundedness)
and \cite{helwaetal_robust_tracking_lagrangian} (considering tracking of Lagrangian systems).

The control performance typically degrades with larger uncertainty sets, and it might even be impossible to achieve
the control goal if the uncertainty sets are too large \cite{skogestad_postlethwaite}.
Therefore it is desirable to extract uncertainty sets that are as small as possible,
but still include the ground truth with a given high probability. In particular, when using GPR,
this necessitates frequentist uncertainty bounds that are not too conservative, 
so that the uncertainty sets are not too large. 
Furthermore, we also have to be able to evaluate the uncertainty bounds numerically,
since robust control methods typically need an explicit uncertainty set.

For learning-based control applications using GPR together with robust control methods, bounds of the following form
have been identified as the most useful ones: Let $X$ be an arbitrary input set and assume that $f: X \rightarrow \R$
is the unknown ground truth. Let $\mu_N$ be the posterior mean using a dataset of size $N$, generated from the ground truth,
and let $\delta \in (0,1)$ be given. We need a function $\nu_N(x)$ that can be explicitly evaluated, such that
with probability at least $1-\delta$ with respect to the data generating process, we have
\begin{equation} \label{eq.uncertainty_bound}
 | f(x) - \mu_N(x) | \leq \nu_N(x) \: \forall x \in X.
\end{equation}
We emphasize that
the probability statement is with respect to the noise generating process, and that the underlying ground truth is just some deterministic
function. For a more thorough discussion we refer to \cite[Section~2.5.3]{berkenkamp_safe_exploration_rl}.

\subsection{Related Work} \label{section.related_work}
Frequentist uncertainty bounds for GPR as considered in this work
were originally developed in the literature on bandits. 
To the best of our knowledge, the first result in this direction was \cite[Theorem~6]{srinivas2010},
which is of the form \eqref{eq.uncertainty_bound} with $\nu_N(x)=\beta_N \sigma_N(x)$. Here $\beta_N$ is a constant
that depends on the \emph{maximum information gain}, an information theoretic quantity. 
For common settings, there exist upper bounds on the latter quantity, though these increase with sample size. 
The result from \cite{srinivas2010} has been significantly improved in \cite[Theorem~2]{cg17},
though the latter still uses an upper bound depending on the maximum information gain. 
The importance of \cite[Theorem~6]{srinivas2010} and \cite[Theorem~2]{cg17} for other applications outside
the bandit setting has been recognized early on, in particular in the control community, for example in
\cite{berkenkampetal_safe_controller_opt, berkenkampetal_learning_roa, kolleretal_learning_mpc_exploration, berkenkamp_safe_exploration_rl,
umlauftetal_fb_lin_gp, helwaetal_robust_tracking_lagrangian}.

Unfortunately, both \cite[Theorem~6]{srinivas2010} and \cite[Theorem~2]{cg17} tend to be very conservative,
especially for control applications \cite{berkenkampetal_safe_controller_opt, berkenkamp_safe_exploration_rl}. 
Furthermore, both results rely on an upper bound of the maximum information gain,
which can be difficult to evaluate \cite{srinivas2010}, though asymptotic bounds are available for standard kernels
like the linear, SE or Matern kernel. However, for most control applications, 
these asymptotic bounds are not useful since one requires a concrete numerical bound in the nonasymptotic setting.

For these reasons, previous work utilizing \cite[Theorem~6]{srinivas2010} or \cite[Theorem~2]{cg17} used heuristics.
In the control community, it is common to choose a constant value $\beta_N \equiv \beta$
in an ad-hoc manner, see for example
\cite{berkenkampetal_safe_rl_stability, berkenkampetal_learning_roa, kolleretal_learning_mpc_exploration,
berkenkampetal_safe_controller_opt, helwaetal_robust_tracking_lagrangian}.
This choice does not reflect the asymptotic behaviour of the scaling since the maximum information gain grows
with the number of samples \cite{srinivas2010}\footnote{Sometimes other heuristics are used that ensure that
$\beta_N$ grows with $N$, cf. e.g. \cite{kandasamyetal_high_dim_bo_additive}}.
The problem with such heuristics is that they might work in practice, but \emph{all} theoretical guarantees
that are based on results like \cite[Theorem~2]{cg17} are lost, in particular, safety guarantees like
constraint satisfaction or certain stability notions. 
This is especially problematic since one of the major incentives to use GPR together with such results
is to \emph{provably} ensure properties like stability of a closed loop system. 

Note that in concrete applications one has to make some assumptions on the ground truth at some point.
However, it is not clear at all how a constant scaling $\beta$ as used in these heuristics can be derived from real-world
properties in a principled manner. In contrast to such heuristics, in the original bounds 
\cite[Theorem~6]{srinivas2010}, \cite[Theorem~2]{cg17} every ingredient 
(i.e., desired probability, size of noise, bound on RKHS norm) has a clear interpretation.

It seems to be well-known in the multiarm-bandit literature that, in the present situation, it is possible
to use more empirical bounds than \cite[Theorem~2]{cg17}. Such approaches are conceptually similar to the results that
we will present in Section \ref{section.bounds} below, cf. \cite{abbasi-yadkori_online_learning}
and the recent work \cite{calandrielloetal_gpo_adaptive_sketching}.
However, to the best of our knowledge, results like Theorem 1 are rarely used in applications.
In particular, it seems that no attempts have been made in the control community to use such a-posteriori bounds
in the GPR context.

Model misspecification in the context of GPR has been discussed before in some works.
In \cite{beckersetal_mse_gp} a numerical method for bounding the mean-squared-error of a misspecified GP
is considered, but it relies on a probabilistic setting.
The recent article \cite{wangtuowu_kriging} provides uniform error bounds and deals with misspecified kernels,
but again uses a probabilistic setting and focuses on the noise-free case.

A work with goals similar to ours is \cite{ledereretal_uniform_error_bounds}. The authors recognize and explicitly
discuss some of the problems of \cite[Theorem~2]{cg17} in the context of control. 
However, \cite{ledereretal_uniform_error_bounds}
uses a probabilistic setting, while our work is concerned with a fixed, but unknown underlying target function, and hence our results
are of a worst-case nature. 
As we argued in Section \ref{section.bayesian_frequ}, this is the setting required for robust approaches. 
Finally, the very recent work \cite{maddalenaetal2020} requires bounded noise and does not deal with model misspecification.
\section{Practical and Rigorous Frequentist Uncertainty Bounds} \label{section.bounds}
We now present our main technical contributions. The key observation is that for many applications relying on
frequentist uncertainty bounds, in particular, learning-based control methods, only \emph{a-posteriori}
bounds are needed. This means that the frequentist uncertainty set can be derived \emph{after} the learning process
and hence the concrete realization of the dataset can be used. We take advantage of this fact and modify existing
uncertainty results so that they explicitly depend on the dataset used for learning.
In general no \emph{a-priori} guarantees can be derived from the results we present here, 
but this does not play a role in the present setting. 

The following result, which is a modified version of \cite[Theorem~2]{cg17}, is
fundamental for the rest of the paper.
\begin{theorem} \label{thm.main_nominal}
 Let $D \not = \emptyset$ be a set and $k:D \times D \rightarrow \R$ a positive definite
 kernel with corresponding RKHS $~{(H_k, \|\cdot\|_k)}$ and let $f \in H_k$ with $\| f \|_k \leq B$
 for some $B \geq 0$.
 Let $\mathbb{F} = (\mathcal{F}_n)_{n \in \N}$ be a filtration
 and $(x_n)_{n \in \N}$ a $D$-valued discrete-time stochastic process
  that is predictable w.r.t. $\mathbb{F}$
 and let
 $(\epsilon_n)_n$ be an $\R$-valued stochastic process adapted to $\mathbb{F}$, such that
 $\epsilon_n$ conditioned on $\mathcal{F}_{n-1}$ is $R$-subgaussian with $R\geq 0$.
 Furthermore, define $y_n = f(x_n) + \epsilon_n$ for all $n \geq 1$.
 
 Consider a Gaussian process $g \distr \GP_D(0, k)$ and denote its posterior mean function by $\mu_N$,
 its posterior covariance function by $k_N$ and its posterior variance by $\sigma^2_N(x)\defsym k_N(x,x)$,
 w.r.t. to data $(x_1,y_1),\ldots,(x_N,y_N)$, assuming in the likelihood independent Gaussian noise with
 mean zero and variance $\lambda > 0$. Then for any $\delta \in (0,1)$ with $\bar{\lambda}=\max\{ 1, \lambda \}$ and\footnote{{\color{blue}In the original version, there was an error in this constant: the denominator $\sqrt{\lambda}$ and the factor $\bar{\lambda}/\lambda$ in the determinant were missing. See the supplementary material for details.}}
 \begin{align} \label{eq.beta_main}
 \beta_n & = \beta_n(\delta, B, R, \lambda) \nonumber \\
 	& = B + \frac{R}{\sqrt{\lambda}}\sqrt{\log\left( \determinant(\bar{\lambda}/\lambda\bb{K}_n + \bar{\lambda} \Ident_n) \right) - 2\log(\delta)}
\end{align}
 one has
 \begin{equation}
 \Pp \pbl | \mu_N(x) - f(x) | \leq \beta_N \sigma_N(x) \: \forall N \in \N, x \in D \pbr \geq 1 - \delta.
\end{equation}
\end{theorem}
\begin{proof} (Idea) Essentially identical to the proof of \cite[Theorem~2]{cg17},
however, we do not upper-bound \eqref{eq.beta_main} by the maximum information gain.
Details are given in the supplementary material.
\end{proof}
The key insight is that by appropriate modifications of the proof
of \cite[Theorem~2]{cg17} we obtain a frequentist uncertainty bound for GPR
that fulfills all desiderata from Section~\ref{section.bayesian_frequ}. In the numerical examples below we show that
the bound is often tight enough for practical purposes.

Independent inputs and noise (e.g. deterministic inputs and i.i.d. noise) is a common situation that is simpler than
the setting of Theorem \ref{thm.main_nominal}. In this case the following a-posteriori bound can be derived  that does not depend anymore
on $\log \determinant$.
\begin{proposition} \label{prop.bound_independent_setting}
 Let $D \not = \emptyset$ be a set and $k:D \times D \rightarrow \R$ a positive definite
 kernel with corresponding RKHS $(H_k, \|\cdot\|_k)$ and let $f \in H_k$ with $\| f \|_k \leq B$
 for some $B \geq 0$.
 Let $x_1,\ldots, x_N \in D$ be given and $\epsilon_1, \ldots,\epsilon_N$ be $\R$-valued
 independent $R$-subgaussian random variables.
 Furthermore, define $y_n = f(x_n) + \epsilon_n$ for all $n=1,\ldots,N$.
 
 Consider a Gaussian process $g \distr \GP_D(0, k)$ and denote its posterior mean function by $\mu_N$,
 its posterior covariance function by $k_N$ and its posterior variance by $\sigma^2_N(x)\defsym k_N(x,x)$,
 w.r.t. to data $(x_1,y_1),\ldots,(x_N,y_N)$, assuming in the likelihood independent Gaussian noise with
 mean zero and variance $\lambda \geq 0$. Then for any $\delta \in (0,1)$ with 
 \begin{align*}
 \eta_N(x) & = R \| (\bb{K}_N + \lambda \Ident_N)^{-1} \bb{k}_N(x) \|  \\
  & \times \sqrt{N + 2\sqrt{N}\sqrt{\log\left[\frac{1}{\delta}\right]} + 2\log\left[\frac{1}{\delta}\right]}
\end{align*}
 one has
 \begin{equation*}
 \Pp\pbl | \mu_N(x) - f(x) | \leq B\sigma_N(x) + \eta_N(x) \: \forall x \in D \pbr \geq 1 - \delta.
\end{equation*}
\end{proposition}
\begin{proof}
(Idea) Similar to the proof of Theorem \ref{thm.main_nominal}, but use \cite[Theorem~2.1]{hsuetal_quad_form_subgaussian}
instead of \cite[Theorem~1]{cg17}. Details are provided in the supplementary material.
\end{proof}

\subsection{Using the Nominal Bound} \label{section.using_bounds}
We will now discuss how the previous bounds can be applied. In particular, we will carefully examine potential difficulties
that have been identified in similar settings in the literature before.
\paragraph{Kernel Choice and RKHS Norm Bound} 
Theorem \ref{thm.main_nominal} and Proposition \ref{prop.bound_independent_setting} require an upper bound on the RKHS norm of the target function, i.e., $\|f\|_k\leq B$.
In particular, the target function has to be in the RKHS corresponding to the covariance function used in GPR. 
Since we do not rely on bounds on the maximum information gain, very general kernels can be used together with
Theorem \ref{thm.main_nominal} and Proposition \ref{prop.bound_independent_setting}.
For example, highly customized kernels can be directly
used and no derivation of additional bounds is necessary, as would be the case for \cite[Theorem~6]{srinivas2010} or \cite[Theorem~2]{cg17}.
In particular, our results support the usage of linearly constrained Gaussian Processes \cite{jidlingetal_linearly_constrained_gp,
langehegermann_algorithmically_constrained_gp} and related approaches like \cite{geisttrimpe_gpgp}.
Getting a large enough, yet not too conservative bound on the RKHS norm of the target function can be difficult in general.
However, since the kernel encodes prior knowledge, domain knowledge could be used to arrive at such upper bounds.
Developing general and systematic methods to transform established domain knowledge into bounds on the RKHS norm is left for future work.

\paragraph{Hyperparameters}
As with other bounds for GPR the inference of hyperparameters is critical.
First of all, an inspection of the proof of %
\cite[Theorem~2]{cg17} shows that the nominal noise variance $\lambda$ is independent
of the actual subgaussian noise\footnote{In fact $\lambda$ in the corresponding \cite[Theorem~2]{cg17} has been used as a tuning parameter in \cite{cg17}.} 
(with subgaussian constant $R$). In particular, Theorem \ref{thm.main_nominal} and
Proposition \ref{prop.bound_independent_setting} hold true for any admissible noise variance, though
it is not clear what the optimal choice would be. 

In GPR the precise specification of the covariance function is usually not given.
In the most common situation a certain kernel class is selected, e.g. SE kernels, and the
remaining hyperparameters, e.g. the length-scale of an SE kernel,
is inferred by likelihood optimization or in a fully Bayesian setting by using hyperpriors.
In both cases the results from above do not directly apply since they rely
on \emph{given} correct kernels. 
Incorporating the hyperparameter inference in results like
Theorem \ref{thm.main_nominal} in a principled manner is an important open question for future work.
As a first step into this direction, we provide robustness results in Section \ref{section.model_misspec}.
\paragraph{Computational Complexity}
For many interesting applications the sample sizes are small enough, so that
standard GPR together with Theorem \ref{thm.main_nominal} can be used directly.
Furthermore, for control applications the GP model is often used offline and only derived information is
used in the actual controller which has restricted computational resources and real-time constraints.

There might be cases where standard GPR can be used, but the calculation of $\log \determinant$ is prohibitive.
In such a situation it is possible to use powerful approximation methods for the latter quantity. 
Since quantitative approximation error bounds are available
in this situation, e.g. \cite{hanetal_logdet, dongetal_logdet_gp}, one can simply add a corresponding margin in the definition of $\beta_N$.
Additionally, in some cases the $\log \determinant$ can be efficiently calculated, in particular 
for semi-separable kernels \cite{andersenchen2020},
which play an important role in system identification \cite{chenandersen_semiseparable}.
Note also that the application of Proposition \ref{prop.bound_independent_setting} does not require the computation of $\log \determinant$ terms.

Additionally, due to the very general nature of Theorem \ref{thm.main_nominal} \emph{any} approximation approach for GPR based on subsets of the
training data that does not update the kernel parameters can be used and no modifications of the uncertainty
bounds are necessary. In particular, the maximum variance selection procedure \cite{jainetal_learning_control_gp,kolleretal_learning_mpc_exploration}
is compatible with Theorem \ref{thm.main_nominal} and Proposition \ref{prop.bound_independent_setting}.

Finally, any GPR approximation method that does not update the kernel hyperparameters and gives quantitative approximation error estimates
for the posterior mean and variance can be used with Theorem \ref{thm.main_nominal}, where the respective
quantities have to be adapted according to the method used. 
An example of such an approximation method with error estimates is \cite{hugginsetal_scalable_gp_guarantees}.

\subsection{Bounds under Model Misspecification} \label{section.model_misspec}
As usual in the literature, Theorem \ref{thm.main_nominal} and Proposition \ref{prop.bound_independent_setting}
use the same kernel $k$ for generating the
RKHS containing the ground truth and as a covariance function in GPR.
Obviously, in practice it is unlikely that one gets the kernel of the ground truth exactly right.
However, it turns out that this is not a big problem. We first note a simple, yet interesting fact
that follows immediately from the proof of \cite[Theorem~2]{cg17}. 
\begin{proposition} \label{prop.simple_kernel_misspec}
 Consider the situation of Theorem \ref{thm.main_nominal}, but this time assume that the ground truth $f$
 is from another RKHS $\tilde{H}$. 
 If $\tilde{H} \subseteq H$, 
 and the inclusion $\mathop{id}: \tilde{H} \rightarrow H$ is continuous with operator norm at most 1,
 then the result holds true 
 without any modification.
\end{proposition}
This simple result can easily be used to verify that for many common situations misspecification of
the kernel is not a problem. As an example, we consider the case of the popular isotropic SE kernel.
\begin{proposition} \label{prop.se_misspec_simple}
 Consider the situation of Theorem \ref{thm.main_nominal}, but 
 let $f \in \tilde{H}$, where $\tilde{H}$ is the RKHS corresponding to the SE kernel $\tilde{k}$ 
 (on $\emptyset \not = D \subseteq \R^d$) with length
 scale $0 < \tilde{\gamma}$. Use for the Gaussian Process Regression the SE kernel $k$ with length-scale
 $0 < \gamma \leq \tilde{\gamma}$. Then Theorem \ref{thm.main_nominal} holds true without change\footnote{{\color{blue} Remark: The RKHS norm bound has to be valid for $\tilde{k}$.}}.
\end{proposition}
\begin{proof}
Follows immediately from Proposition \ref{prop.simple_kernel_misspec} and \cite[Proposition~4.46]{svm_book}.
\end{proof}

These results tell us that it is not a problem if we do not get the hyperparameter of the isotropic SE kernel
right, as long as we \emph{underestimate} the length-scale. This is intuitively clear since
a smaller length-scale corresponds to more complex functions. Similar results are known in related contexts,
for example \cite{szaboetal_frequentist_coverage_adaptive}. %
Note that Proposition \ref{prop.se_misspec_simple} does not imply that one should choose a very small length-scale.
The result makes a statement on the validity of the frequentist uncertainty bound from Theorem \ref{thm.main_nominal}, but not on
the size of the uncertainty set. For a similar discussion in a slightly different setting see \cite{wangtuowu_kriging}.

Next, we consider the question what happens when the ground truth is from a different kernel $\tilde{k}$,
such that the corresponding RKHS $\tilde{H}$ is \emph{not} included in the RKHS $H$ corresponding to the
covariance function $k$ used in the GPR. Intuitively, since the functions in an RKHS
are built from the corresponding reproducing kernel, cf. e.g. \cite[Theorem~4.21]{svm_book}, Theorem \ref{thm.main_nominal}
should hold approximately true if the kernels $\tilde{k}$ and $k$ are not too different.
This is made precise in the following result. Its proof is provided in the supplementary material.
\begin{theorem} \label{thm.kernel_misspec} 
 Consider the situation of Theorem \ref{thm.main_nominal} but this time assume that the target function $f$ is from 
 the RKHS $(\tilde{H}, \|\cdot\|_{\tilde{k}})$ of a different kernel $\tilde{k}$ such that
 still $\|f\|_{\tilde{k}} \leq B$ and 
 $\sup_{x,x^\prime \in D} |k(x,x^\prime)-\tilde{k}(x,x^\prime)| \leq \tilde{\epsilon}$  for some $\tilde{\epsilon}\geq 0$.
 We then have for any $\delta \in (0,1)$ that
 \begin{equation*}
 \Pp\pbl | \mu_N(x) - f(x) | \leq \nu_N(x)  \: \forall N \in \N, x \in D \pbr \geq 1 - \delta
\end{equation*}
where
\begin{equation*}
	\nu_N(x) =  \bar{\beta}_N\sqrt{\sigma^2_N(x) + S^2_N(x)}  + C_N(x)\|\bb{y}_N\|,
\end{equation*}
 and\footnote{{\color{blue}Since this constant is based on Theorem \ref{thm.main_nominal}, the expression has been adapted accordingly in this version.}} %
 {\small
 \begin{equation*}
	\bar{\beta}_N = B +  \frac{R}{\sqrt{\lambda}}\sqrt{\log \determinant \left(\frac{\bar{\lambda}}{\lambda} \bb{K}_N + \left(\frac{\bar{\lambda}}{\lambda}N\tilde{\epsilon} + \bar\lambda\right) \Ident_N \right) - 2\log(\delta)}
 \end{equation*} }
 and
 \begin{align}
  S^2_N(x) & = \tilde{\epsilon} + \sqrt{N} \tilde{\epsilon}  \| (\bb{K}_N + \lambda \Ident_N )^{-1}\bb{k}_N(x) \| \nonumber \\
    & \hspace{1cm} +  (\sqrt{N} \tilde{\epsilon} +  \| \bb{k}_N(x) \|) C_N(x) \label{eq.S} \\
  C_N(x) & = \left( \frac{1}{\lambda} + \|  (\bb{K}_N + \lambda \Ident_N)^{-1} \| \right)(\| \bb{k}_N(x) \| + \sqrt{N} \tilde{\epsilon}) \nonumber \\
    & \hspace{1cm} + \| (\bb{K}_N - \lambda \Ident_N)^{-1} \| \sqrt{N} \tilde{\epsilon} \label{eq.C}
 \end{align}
\end{theorem}
We can interpret Theorem \ref{thm.kernel_misspec} as a robust version (w.r.t. disturbance of the kernel)
of Theorem \ref{thm.main_nominal}. By using the bounds from Theorem \ref{thm.kernel_misspec} we can ensure that the resulting
uncertainty set contains the ground truth with prescribed probability despite using the wrong kernel.

In Theorem \ref{thm.main_nominal}
we have a tube around $\mu_N(x)$ of width $\beta_N(\delta, B, R, \lambda)\sigma_N(x)$, whereas in Theorem \ref{thm.kernel_misspec}
we have a tube around $\mu_N(x)$ of width
\begin{equation*}
  \bar{\beta}_N \sqrt{\sigma^2_N(x) + S^2_N(x)} + C_N(x)\|\bb{y}_N\| 
\end{equation*}
Because of the uncertainty or disturbance
in the kernel we have to increase the nominal noise variance (used in the nominal noise model in GPR),
increase the nominal posterior standard deviation
from $\sigma_N(x)$ to $\sqrt{\sigma^2_N(x) + S^2_N(x)}$
and add an offset to the width of the tube of 
$C_N(x)\|\bb{y}_N\|$. In particular, the uncertainty set now depends on the measured values $\bb{y}_N$.
Note that $C_N(x)$ and $S_N(x)$ depend on the input, but if necessary this dependence can be easily removed by finding an upper bound on
$\|\bb{k}_N(x)\|$.
An interesting observation is that even if $\sigma_N(x)=0$, the width of the tube around $\mu_N(x)$
has nonzero width. Intuitively this is clear since in general
it can happen that $f \not \in H$, but $\mu_N \in H$ by construction.
Finally, a robust version of Proposition \ref{prop.bound_independent_setting} can be derived similarly to
Theorem \ref{thm.kernel_misspec}, see the supplementary material.
\section{Numerical Experiments} \label{section.experiments}
We now test the theoretical results in numerical experiments using synthetic data, where the ground truth is known.
First, we investigate the frequentist behaviour of the uncertainty bounds.
The general setup consists of generating a function from an RKHS with known RKHS norm (the ground truth),
sampling a finite data set, running GPR on this data set and then evaluating our uncertainty bounds.
This learning process is then repeated on many sampled data sets and we check how often the uncertainty bounds are violated. 
To the best of our knowledge, these types of experiments have actually not
been done before, since usually only the final algorithms using the uncertainty bounds have been tested, e.g. in \cite{srinivas2010,cg17},
or no ground truth from an RKHS has been used, cf. \cite{berkenkamp_safe_exploration_rl}.
Furthermore, we would like to remark that the method for generating the ground truth can result in misleading judgements of the
theoretical result. The latter has to hold for all ground truths, but the generating method can introduce a bias, e.g. only generating functions
of a certain shape. To the best of our knowledge, this issue has not been pointed out before.

Second, we demonstrate the practicality and usefulness of our uncertainty bounds by applying them to concrete example from robust control.
\subsection{Frequentist Behavior of Uncertainty Bounds}
\paragraph{Setup of the Experiments}
Unless noted otherwise, for each of the following experiments we use $D=[-1,1]$ as input space and
generate randomly 50 functions (the ground truths) from an RKHS and evaluate these on a grid 1000 equidistant points from $D$.
We generate the functions by randomly selecting a number of center points $x \in D$ and form
a linear combination of kernel functions $k(\cdot,x)$, where $k$ is the kernel of the RKHS.
Additionally, for the SE kernel we use the orthonormal basis (ONB) from \cite[Section~4.4]{svm_book} with random, normally distributed coefficients
to generate functions from the corresponding RKHS.
For each function we repeat the following learning instance 10000 times: We sample uniformly 50 inputs, 
evaluate the ground truth on these inputs and adding normal zero-mean i.i.d. noise with standard deviation (SD) 0.5. 
We then run GPR on each of the training sets, compute the uncertainty sets (for Theorem \ref{thm.main_nominal} this corresponds to the scaling $\beta_{50}$)
and check (on the equidistant grid of $D$) whether the resulting uncertainty set contains the ground truth.
\paragraph{Nominal Setting}
Consider the case that the covariance function used in GPR and the kernel corresponding to the RKHS of the target function coincide.
First, we test the nominal bound from Theorem \ref{thm.main_nominal} with SE and Matern kernels, respectively,
for $\delta=0.1,0.01,0.001,0.0001$. 
The resulting scalings $\beta_{50}$ are shown in Table \ref{tab.nominal}.
As can be seen there, the scalings are reasonably small, roughly in the range of heuristics used in the  literature, cf.  \cite{berkenkamp_safe_exploration_rl}.
Furthermore, the graph of the target function is fully included in the uncertainty set in all repetitions.
This is illustrated in Figure \ref{fig.illustration} (LEFT), where an example instance of this experiment is shown. 
As can be clearly seen there, the posterior mean (blue solid line) is well within the uncertainty set from Theorem \ref{thm.main_nominal} (with $\delta = 0.01$), 
which is not overly conservative.
\begin{figure}
	\includegraphics[width=0.23\textwidth]{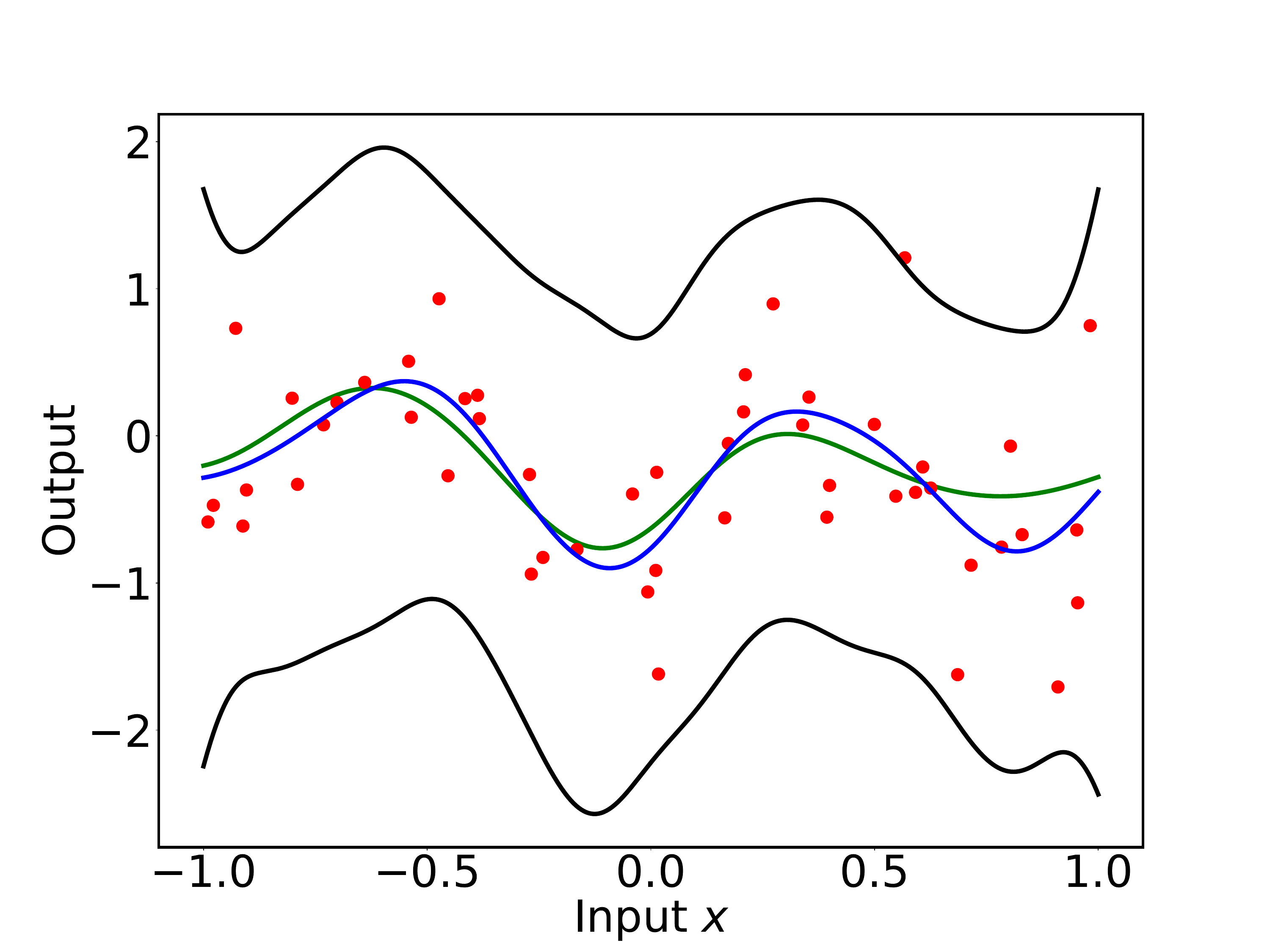}
	\includegraphics[width=0.23\textwidth]{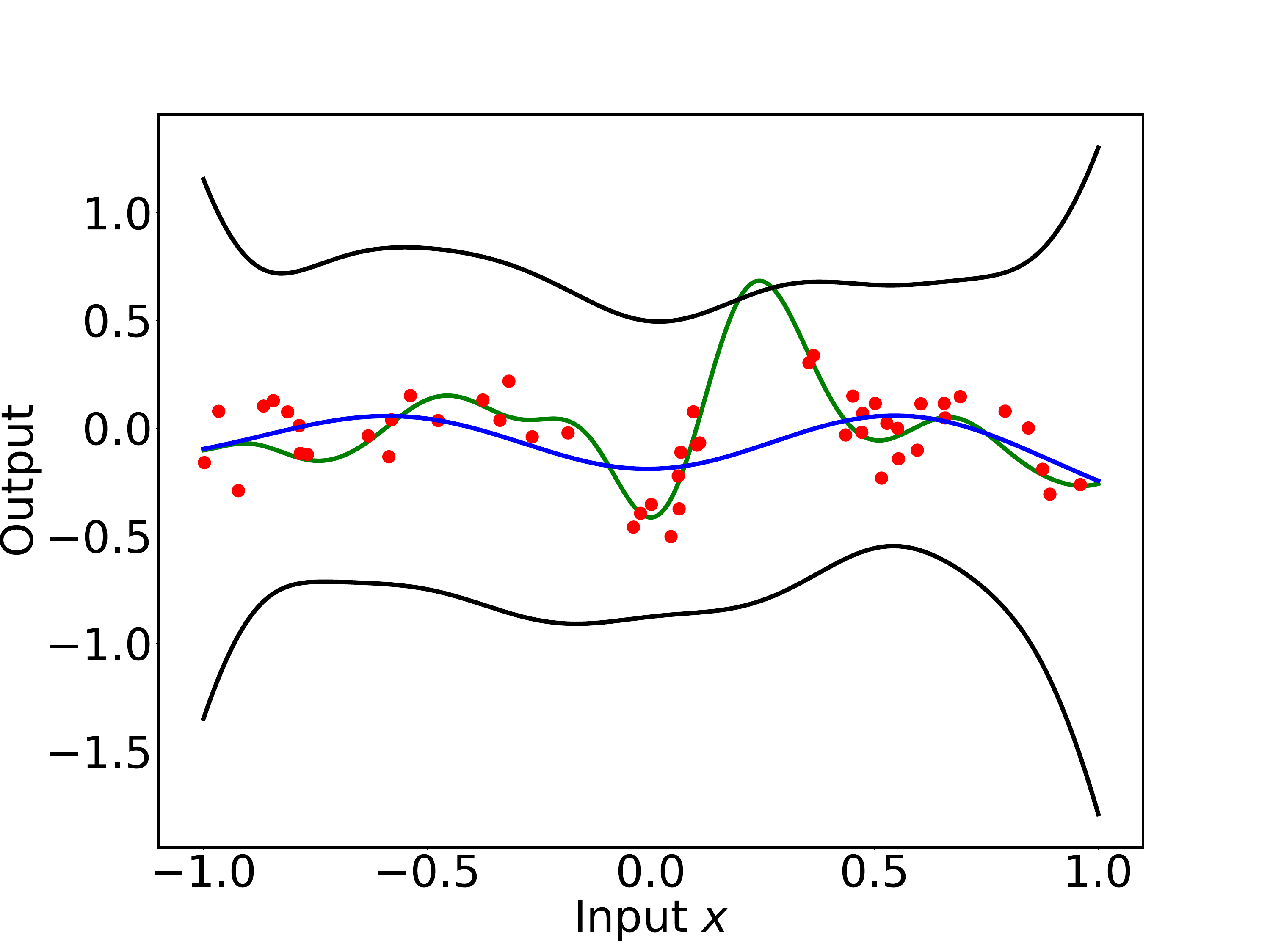}
	\caption{LEFT (Nominal setting): Example function (green) from SE kernel with length-scale 0.5 and RKHS norm 2, learned from 50 samples (red).
	Shown is the posterior mean (blue) and the uncertainty set from Theorem \ref{thm.main_nominal} for $\delta=0.01$. 
	RIGHT (Misspecified setting): Example function from SE kernel with length-scale 0.2, learned with GPR using 
	SE covariance function with length-scale 0.5.
	The violation of the uncertainty set is clearly visible.}
	\label{fig.illustration}
\end{figure}
\paragraph{Misspecified Setting}
We now consider misspecification of the kernel, i.e. different kernels are used for generating the ground truth and as prior covariance function in GPR.
As an example we use the SE kernel with different length-scales and use the ONB from  \cite[Section~4.4]{svm_book} to generate the
RKHS functions.
We start with the \emph{benign} setting from Proposition \ref{prop.simple_kernel_misspec}, where the RKHS corresponding to the covariance
function used in GPR contains the RKHS of the target function. For this, identical settings as in our first experiment above are used,
but now we generate functions from the SE kernel with length-scale 0.5  and use the SE kernel with length-scale 0.2 in GPR. This 
corresponds to the benign setting according to Proposition \ref{prop.se_misspec_simple}.
As expected, we find the same results as above, i.e. the uncertainty set fully contains the ground truth
in all repetitions. Furthermore, the scalings $\beta_{50}$ are roughly of the same size as in the nominal case, cf. Table \ref{tab.misspec} (upper row).

Next, we investigate the \emph{problematic} setting where the RKHS corresponding to the covariance
function used in GPR does not contain the RKHS of the target function anymore. As an example we use again the SE kernel, but now
with length-scale 0.2 for generating RKHS functions and the SE kernel with length-scale 0.5 for GPR. 
We found a considerable number of function instances where the bounds from Theorem \ref{thm.main_nominal} were violated
with higher frequency than $\delta$. More precisely, for a given function the tube of width $\beta_{50} \sigma_{50}(x)$
around $\mu_{50}(x)$ does not fully contain the function in more than $\delta \times 10000$ of the learning instances.
This happened for 2,  6, 12, 13 out of 50 functions for $\delta=0.1,0.01,0.001,0.0001$, respectively.

Interestingly, when performing this experiment using the standard approach of generating functions from the RKHS
based on linear combinations of kernels, we did not find functions that violated the uncertainty bounds more often than prescribed.
This reaffirmes our introductory remark that the method generating the test targets can lead to wrong judgements of the theoretical results.

The results of the previous two experiments indicate that a model misspecification of the kernel can be a problem and a robust result like
Theorem \ref{thm.kernel_misspec} is necessary. Indeed, a repetition of the last experiment with the uncertainty set from Theorem \ref{thm.kernel_misspec}
instead of Theorem \ref{thm.main_nominal} resulted in all functions being contained in the uncertainty set in all repetitions.
An inspection of the average uncertainty set widths
in Table \ref{tab.misspec_robust} indicates some conservatism.
\begin{table}
 \caption{$\beta_{50}$ in nominal setting (mean $\pm$ standard deviation over all repetitions)}
 \label{tab.nominal}
 \centering \scriptsize
 \begin{tabular}{l l l l l}
  $\delta$ & 0.1		& 0.01			& 0.001 		& 0.0001 \\
  \hline
  SE 	   & $6.95 \pm 0.04$	& $7.39 \pm 0.04$	& $7.80 \pm 0.03$	& $8.19 \pm 0.03$ \\
  Matern   & $7.36 \pm 0.04$	& $7.78 \pm 0.04$	& $8.16 \pm 0.04$	& $8.53 \pm 0.04$\\
 \end{tabular}
\end{table}

\begin{table}
	\caption{$\beta_{50}$ in the misspecified setting  (mean $\pm$ standard deviation over all repetitions)}
	\label{tab.misspec}
	 \centering \scriptsize
 \begin{tabular}{l l l l l}
   $\delta$ 		& 0.1		& 0.01			& 0.001 		& 0.0001 \\
	\hline
   Benign 	   	& $6.53 \pm 0.038$	& $6.97 \pm 0.035$	& $7.39 \pm 0.033$	& $7.77 \pm 0.031$ \\
   Problematic   	& $6.11 \pm 0.03$	& $6.64 \pm 0.03$	& $7.11 \pm 0.02$	& $7.54 \pm 0.02$\\
  \end{tabular}
\end{table}

\begin{table}
 \caption{Width of robust uncertainty set (mean $\pm$ SD of average width)}
 \label{tab.misspec_robust}
 \centering \scriptsize
 \begin{tabular}{l l l l l}
  $\delta$ & 0.1		& 0.01			& 0.001 		& 0.0001 \\
 \hline
  Mean & $71.68 \pm 5.36$  & $73.79 \pm 5.36$  & $75.64 \pm 5.37$  & $77.33 \pm 5.37$ \\
  SD & $6.54 \pm 1.73$  & $6.73 \pm 1.78$  & $6.91 \pm 1.82$  & $7.06 \pm 1.86$ 
 \end{tabular}
\end{table}

\subsection{Control Example}
We now show the usefulness of our results for robust control by applying it to a concrete, existing
learning-based control method. Due to space constraints only a brief description of the example can be given here.
For more details and discussions we refer to the supplementary material.

As an example, we choose the algorithm from \cite{solopertoetal_learning_rmpc_gp}
which is a learning-based Robust Model Predictive Control (RMPC) approach that comes with rigorous
control-theoretic guarantees. We refer to \cite[Chapter~3]{RMD} for background on RMPC and
to \cite{hewingetal_learning_mpc_review} for a recent survey on related learning-based control methods.
We follow \cite{solopertoetal_learning_rmpc_gp} and consider the discrete-time system
\begin{equation} \label{eq.system} \small
	\begin{bmatrix}
		x_1^+ \\
		x_2^+
	\end{bmatrix}
	=
	\begin{bmatrix}
		0.995 & 0.095\\
		-0.095 & 0.900
	\end{bmatrix}
	\begin{bmatrix}
		x_1 \\
		x_2
	\end{bmatrix}
	+
	\begin{bmatrix}
		0.048 \\
		0.95
	\end{bmatrix}
	u
	+
	\begin{bmatrix}
		0 \\
		-r(x_2)
	\end{bmatrix}
\end{equation}
modelling a mass-spring-damper system with some nonlinearity $r$ (this could be interpreted as a friction term). 
The goal is the stabilization of the origin subject to the state and control constraints $\mathbb{X}=[-10,10] \times [-10,10]$ and $\mathbb{U}=[-3,3]$,
as well as minimizing a quadratic cost. 

The approach from \cite{solopertoetal_learning_rmpc_gp}  performs this task by interpreting \eqref{eq.system} 
as a linear system with disturbance, given by the nonlinearity $r$, whose graph is a-priori known to lie in the
set $\mathbb{W}_0 = [-10,10] \times [-7,7]$. 
The nonlinearity is assumed to be unknown and has to be learned from data. The RMPC algorithm
requires as an input disturbance sets $\mathbb{W}(x)$ such that $\begin{pmatrix}0 &-r(x_2)\end{pmatrix}^\top \in \mathbb{W}(x)$ for all $x \in \mathbb{X}$,
which are in turn used to generate tightened nominal constraints ensuring robust constraint satisfaction.
Furthermore, the tighter the sets $\mathbb{W}(x)$ are, the better is the performance of the algorithm.
We now randomly generate the function $r$ (which will be our ground truth) from the RKHS with kernel
$k(x,x^\prime) = 4\exp \left( - \frac{(x-x^\prime)^2}{2 \times 0.8^2} \right)$
with RKHS norm 2.
Following \cite{solopertoetal_learning_rmpc_gp}, we uniformly sample
100 partial states $x_2 \in [-10,10]$, evaluate $r$ at these and add i.i.d. Gaussian noise with a standard deviation of 0.01 to it.
The unknown function is then learned using GPR from this data set. Our results then lead to an uncertainty set of the form
$\mathbb{W}(x)=[\mu_{100}(x_2) - \beta_{100} \sigma_{100}(x_2), \mu_{100}(x_2) + \beta_{100} \sigma_{100}(x_2)],$
with $\beta_{100}$ from Theorem \ref{thm.main_nominal} for $\delta:=0.001$.
In particular, with probability at least $1-\delta$ we can guarantee that $r(x_2) \in \mathbb{W}(x)$ holds for all $x \in \mathbb{X}$. 
\begin{figure}
	\includegraphics[width=0.23\textwidth]{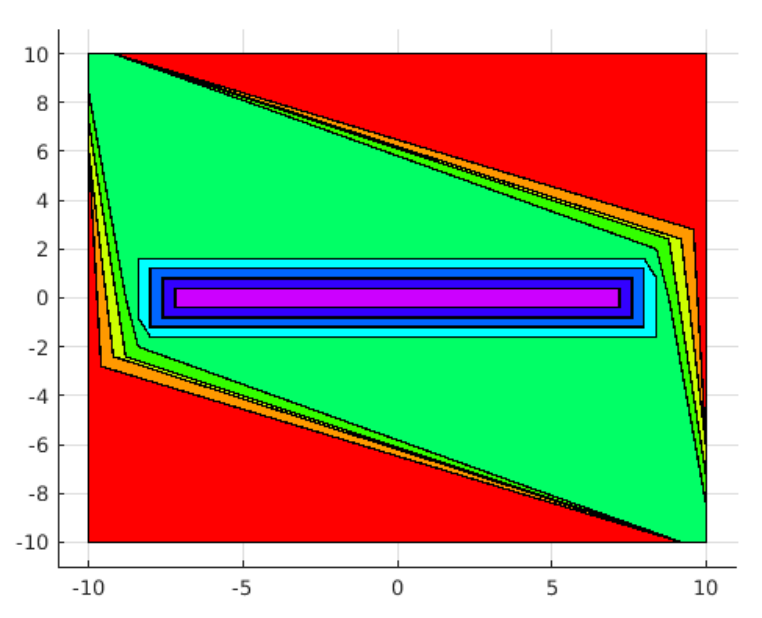}
	\includegraphics[width=0.23\textwidth]{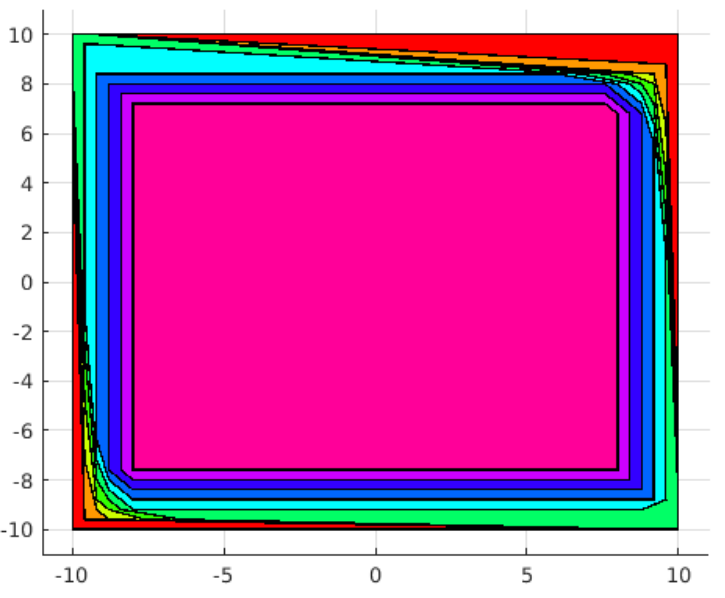}
	\caption{Tightened state constraint sets $\mathbb{Z}_k$ for $k=0,\ldots,9$. Computed from a-priori uncertainty set $\mathbb{W}_0$ (LEFT)
	and learned uncertainty sets $\mathbb{W}(x)$ (RIGHT).}
	\label{fig.constraints}
\end{figure}
Figure \ref{fig.constraints} shows the resulting tightened state constraints for an MPC horizon of 9. Clearly, the state constraint sets from the learned
uncertainty sets are much larger. Furthermore, in contrast to previous work, we can guarantee that the RMPC controller using these tightened state constraints
retains all control-theoretic guarantees with probability at least $1-\delta$. In the present case we can ensure state and control constraint satisfaction, input-to-state stability and convergence to a neighborhood of the origin, with probability at least $1-\delta$. This follows immediately from
\cite[Theorem~1]{solopertoetal_learning_rmpc_gp}, since the ground truth $r$ is covered by the uncertainty sets $\mathbb{W}(x)$ with this probability.
Note that no changes to the existing RMPC scheme were necessary and the control-theoretic guarantees were retained with
prescribed high probability.
\section{Conclusion} \label{section.conclusion}
We discussed the importance of frequentist uncertainty bounds for Gaussian Process Regression and improved
existing results in this kind. By aiming only at a-posteriori bounds we were able to provide
rigorous \emph{and} practical uncertainty results. Our bounds can be explicitly evaluated and are sharp enough to be useful
for concrete applications, as demonstrated with numerical experiments. Furthermore, we also introduced robust
versions that work despite certain model mismatches, which is an important concern in real-world applications.
We see the present work as a starting point for further developments, in particular, domain-specific versions
of our results and more specific and less conservative robustness results.

\subsubsection{Acknowledgments.}
We would like to thank Sayak Ray Chowdhury and Tobias Holicki for helpful discussions, 
Steve Heim for useful comments on a draft of this work,
and Raffaele Soloperto for providing the Matlab code from \cite{solopertoetal_learning_rmpc_gp}.
Furthermore, we would like to thank the reviewers for their helpful and constructive feedback.
Funded by Deutsche Forschungsgemeinschaft (DFG, German Research Foundation) under Germany's Excellence Strategy - EXC 2075 – 390740016 and in part by the Cyber Valley Initiative. We acknowledge the support by the Stuttgart Center for Simulation Science (SimTech).

\bibliography{refs_main}

\begin{thebibliography}{40}
\providecommand{\natexlab}[1]{#1}
\providecommand{\url}[1]{\texttt{#1}}
\providecommand{\urlprefix}{URL }
\expandafter\ifx\csname urlstyle\endcsname\relax
  \providecommand{\doi}[1]{doi:\discretionary{}{}{}#1}\else
  \providecommand{\doi}{doi:\discretionary{}{}{}\begingroup
  \urlstyle{rm}\Url}\fi

\bibitem[{Abbasi-Yadkori(2013)}]{abbasi-yadkori_online_learning}
Abbasi-Yadkori, Y. 2013.
\newblock \emph{Online learning for linearly parametrized control problems}.
\newblock Ph.D. thesis, University of Alberta.

\bibitem[{Andersen and Chen(2020)}]{andersenchen2020}
Andersen, M.~S.; and Chen, T. 2020.
\newblock Smoothing Splines and Rank Structured Matrices: Revisiting the Spline
  Kernel.
\newblock \emph{SIAM Journal on Matrix Analysis and Applications} 41(2):
  389--412.

\bibitem[{{\AA}str{\"o}m and Murray(2010)}]{astrom_murray}
{\AA}str{\"o}m, K.~J.; and Murray, R.~M. 2010.
\newblock \emph{Feedback systems: an introduction for scientists and
  engineers}.
\newblock Princeton university press.

\bibitem[{Beckers, Umlauft, and Hirche(2018)}]{beckersetal_mse_gp}
Beckers, T.; Umlauft, J.; and Hirche, S. 2018.
\newblock Mean square prediction error of misspecified Gaussian process models.
\newblock In \emph{2018 IEEE Conference on Decision and Control (CDC)},
  1162--1167. IEEE.

\bibitem[{Berkenkamp(2019)}]{berkenkamp_safe_exploration_rl}
Berkenkamp, F. 2019.
\newblock \emph{Safe Exploration in Reinforcement Learning: Theory and
  Applications in Robotics}.
\newblock Ph.D. thesis, ETH Zurich.

\bibitem[{Berkenkamp et~al.(2016)Berkenkamp, Moriconi, Schoellig, and
  Krause}]{berkenkampetal_learning_roa}
Berkenkamp, F.; Moriconi, R.; Schoellig, A.~P.; and Krause, A. 2016.
\newblock Safe learning of regions of attraction for uncertain, nonlinear
  systems with gaussian processes.
\newblock In \emph{2016 IEEE 55th Conference on Decision and Control (CDC)},
  4661--4666. IEEE.

\bibitem[{Berkenkamp, Schoellig, and
  Krause(2016)}]{berkenkampetal_safe_controller_opt}
Berkenkamp, F.; Schoellig, A.~P.; and Krause, A. 2016.
\newblock Safe controller optimization for quadrotors with Gaussian processes.
\newblock In \emph{2016 IEEE International Conference on Robotics and
  Automation (ICRA)}, 491--496. IEEE.

\bibitem[{Berkenkamp et~al.(2017)Berkenkamp, Turchetta, Schoellig, and
  Krause}]{berkenkampetal_safe_rl_stability}
Berkenkamp, F.; Turchetta, M.; Schoellig, A.; and Krause, A. 2017.
\newblock Safe model-based reinforcement learning with stability guarantees.
\newblock In \emph{Advances in neural information processing systems},
  908--918.

\bibitem[{Berlinet and Thomas-Agnan(2011)}]{berlinet_kernel}
Berlinet, A.; and Thomas-Agnan, C. 2011.
\newblock \emph{Reproducing kernel Hilbert spaces in probability and
  statistics}.
\newblock Springer Science \& Business Media.

\bibitem[{Calandriello et~al.(2019)Calandriello, Carratino, Lazaric, Valko, and
  Rosasco}]{calandrielloetal_gpo_adaptive_sketching}
Calandriello, D.; Carratino, L.; Lazaric, A.; Valko, M.; and Rosasco, L. 2019.
\newblock Gaussian Process Optimization with Adaptive Sketching: Scalable and
  No Regret.
\newblock In \emph{Conference on Learning Theory}, 533--557.

\bibitem[{Chen and Andersen(2020)}]{chenandersen_semiseparable}
Chen, T.; and Andersen, M. 2020.
\newblock On Semiseparable Kernels and Efficient Computation of Regularized
  System Identification and Function Estimation.
\newblock \emph{IFAC-V 2020} .

\bibitem[{Chowdhury and Gopalan(2017)}]{cg17}
Chowdhury, S.~R.; and Gopalan, A. 2017.
\newblock On Kernelized Multi-armed Bandits.
\newblock In \emph{International Conference on Machine Learning}, 844--853.

\bibitem[{Dong et~al.(2017)Dong, Eriksson, Nickisch, Bindel, and
  Wilson}]{dongetal_logdet_gp}
Dong, K.; Eriksson, D.; Nickisch, H.; Bindel, D.; and Wilson, A.~G. 2017.
\newblock Scalable log determinants for Gaussian process kernel learning.
\newblock In \emph{Advances in Neural Information Processing Systems},
  6327--6337.

\bibitem[{Geist and Trimpe(2020)}]{geisttrimpe_gpgp}
Geist, A.~R.; and Trimpe, S. 2020.
\newblock Learning Constrained Dynamics with Gauss Principle adhering Gaussian
  Processes.
\newblock In \emph{Proceedings of the 2nd Conference on Learning for Dynamics
  and Control}, 225--234.

\bibitem[{Han, Malioutov, and Shin(2015)}]{hanetal_logdet}
Han, I.; Malioutov, D.; and Shin, J. 2015.
\newblock Large-scale log-determinant computation through stochastic Chebyshev
  expansions.
\newblock In \emph{International Conference on Machine Learning}, 908--917.

\bibitem[{Helwa, Heins, and
  Schoellig(2019)}]{helwaetal_robust_tracking_lagrangian}
Helwa, M.~K.; Heins, A.; and Schoellig, A.~P. 2019.
\newblock Provably robust learning-based approach for high-accuracy tracking
  control of lagrangian systems.
\newblock \emph{IEEE Robotics and Automation Letters} 4(2): 1587--1594.

\bibitem[{Hewing et~al.(2019)Hewing, Wabersich, Menner, and
  Zeilinger}]{hewingetal_learning_mpc_review}
Hewing, L.; Wabersich, K.~P.; Menner, M.; and Zeilinger, M.~N. 2019.
\newblock Learning-Based Model Predictive Control: Toward Safe Learning in
  Control.
\newblock \emph{Annual Review of Control, Robotics, and Autonomous Systems} 3.

\bibitem[{Hsu et~al.(2012)Hsu, Kakade, Zhang
  et~al.}]{hsuetal_quad_form_subgaussian}
Hsu, D.; Kakade, S.; Zhang, T.; et~al. 2012.
\newblock A tail inequality for quadratic forms of subgaussian random vectors.
\newblock \emph{Electronic Communications in Probability} 17.

\bibitem[{Huggins et~al.(2019)Huggins, Campbell, Kasprzak, and
  Broderick}]{hugginsetal_scalable_gp_guarantees}
Huggins, J.~H.; Campbell, T.; Kasprzak, M.; and Broderick, T. 2019.
\newblock Scalable Gaussian Process Inference with Finite-data Mean and
  Variance Guarantees.
\newblock In \emph{The 22nd International Conference on Artificial Intelligence
  and Statistics}, 796--805.

\bibitem[{Jain et~al.(2018)Jain, Nghiem, Morari, and
  Mangharam}]{jainetal_learning_control_gp}
Jain, A.; Nghiem, T.; Morari, M.; and Mangharam, R. 2018.
\newblock Learning and control using Gaussian processes.
\newblock In \emph{2018 ACM/IEEE 9th International Conference on Cyber-Physical
  Systems (ICCPS)}, 140--149. IEEE.

\bibitem[{Jidling et~al.(2017)Jidling, Wahlstr{\"o}m, Wills, and
  Sch{\"o}n}]{jidlingetal_linearly_constrained_gp}
Jidling, C.; Wahlstr{\"o}m, N.; Wills, A.; and Sch{\"o}n, T.~B. 2017.
\newblock Linearly constrained Gaussian processes.
\newblock In \emph{Advances in Neural Information Processing Systems},
  1215--1224.

\bibitem[{Kanagawa et~al.(2018)Kanagawa, Hennig, Sejdinovic, and
  Sriperumbudur}]{kanagawaetal_gp_kernels}
Kanagawa, M.; Hennig, P.; Sejdinovic, D.; and Sriperumbudur, B.~K. 2018.
\newblock Gaussian processes and kernel methods: A review on connections and
  equivalences.
\newblock \emph{arXiv preprint arXiv:1807.02582} .

\bibitem[{Kandasamy, Schneider, and
  P{\'o}czos(2015)}]{kandasamyetal_high_dim_bo_additive}
Kandasamy, K.; Schneider, J.; and P{\'o}czos, B. 2015.
\newblock High dimensional Bayesian optimisation and bandits via additive
  models.
\newblock In \emph{International Conference on Machine Learning}, 295--304.

\bibitem[{Kocijan(2016)}]{kocijan_modelling_control_gp}
Kocijan, J. 2016.
\newblock \emph{Modelling and control of dynamic systems using Gaussian process
  models}.
\newblock Springer.

\bibitem[{Koller et~al.(2018)Koller, Berkenkamp, Turchetta, and
  Krause}]{kolleretal_learning_mpc_exploration}
Koller, T.; Berkenkamp, F.; Turchetta, M.; and Krause, A. 2018.
\newblock Learning-based model predictive control for safe exploration.
\newblock In \emph{2018 IEEE Conference on Decision and Control (CDC)},
  6059--6066. IEEE.

\bibitem[{Lange-Hegermann(2018)}]{langehegermann_algorithmically_constrained_gp}
Lange-Hegermann, M. 2018.
\newblock Algorithmic linearly constrained Gaussian processes.
\newblock In \emph{Advances in Neural Information Processing Systems},
  2137--2148.

\bibitem[{Lederer, Umlauft, and
  Hirche(2019)}]{ledereretal_uniform_error_bounds}
Lederer, A.; Umlauft, J.; and Hirche, S. 2019.
\newblock Uniform Error Bounds for Gaussian Process Regression with Application
  to Safe Control.
\newblock In \emph{Advances in Neural Information Processing Systems},
  657--667.

\bibitem[{Liu et~al.(2018)Liu, Chowdhary, Da~Silva, Liu, and
  How}]{liu_control_gp_tutorial}
Liu, M.; Chowdhary, G.; Da~Silva, B.~C.; Liu, S.-Y.; and How, J.~P. 2018.
\newblock Gaussian processes for learning and control: A tutorial with
  examples.
\newblock \emph{IEEE Control Systems Magazine} 38(5): 53--86.

\bibitem[{Maddalena, Scharnhorst, and Jones(2020)}]{maddalenaetal2020}
Maddalena, E.~T.; Scharnhorst, P.; and Jones, C.~N. 2020.
\newblock Deterministic error bounds for kernel-based learning techniques under
  bounded noise.
\newblock \emph{arXiv preprint arXiv:2008.04005} .

\bibitem[{Murphy(2012)}]{murphy_ml}
Murphy, K.~P. 2012.
\newblock \emph{Machine learning: a probabilistic perspective}.
\newblock MIT press.

\bibitem[{Rasmussen and Williams(2006)}]{rasmussen_williams_gp}
Rasmussen, C.~E.; and Williams, C.~K. 2006.
\newblock \emph{Gaussian Processes for Machine Learning}.
\newblock The MIT Press.

\bibitem[{Rawlings, Mayne, and Diehl(2017)}]{RMD}
Rawlings, J.~B.; Mayne, D.~Q.; and Diehl, M. 2017.
\newblock \emph{Model predictive control: theory, computation, and design},
  volume~2.
\newblock Nob Hill Publishing Madison, WI.

\bibitem[{Shahriari et~al.(2015)Shahriari, Swersky, Wang, Adams, and
  De~Freitas}]{shahriarietal_review_bo}
Shahriari, B.; Swersky, K.; Wang, Z.; Adams, R.~P.; and De~Freitas, N. 2015.
\newblock Taking the human out of the loop: A review of Bayesian optimization.
\newblock \emph{Proceedings of the IEEE} 104(1): 148--175.

\bibitem[{Skogestad and Postlethwaite(2007)}]{skogestad_postlethwaite}
Skogestad, S.; and Postlethwaite, I. 2007.
\newblock \emph{Multivariable feedback control: analysis and design}, volume~2.
\newblock Wiley New York.

\bibitem[{Soloperto et~al.(2018)Soloperto, M{\"u}ller, Trimpe, and
  Allg{\"o}wer}]{solopertoetal_learning_rmpc_gp}
Soloperto, R.; M{\"u}ller, M.~A.; Trimpe, S.; and Allg{\"o}wer, F. 2018.
\newblock Learning-based robust model predictive control with state-dependent
  uncertainty.
\newblock \emph{IFAC-PapersOnLine} 51(20): 442--447.

\bibitem[{Srinivas et~al.(2010)Srinivas, Krause, Kakade, and
  Seeger}]{srinivas2010}
Srinivas, N.; Krause, A.; Kakade, S.; and Seeger, M. 2010.
\newblock Gaussian process optimization in the bandit setting: no regret and
  experimental design.
\newblock In \emph{Proceedings of the 27th International Conference on
  International Conference on Machine Learning}, 1015--1022.

\bibitem[{Steinwart and Christmann(2008)}]{svm_book}
Steinwart, I.; and Christmann, A. 2008.
\newblock \emph{Support vector machines}.
\newblock Springer Science \& Business Media.

\bibitem[{Szab{\'o} et~al.(2015)Szab{\'o}, Van Der~Vaart, van Zanten
  et~al.}]{szaboetal_frequentist_coverage_adaptive}
Szab{\'o}, B.; Van Der~Vaart, A.~W.; van Zanten, J.; et~al. 2015.
\newblock Frequentist coverage of adaptive nonparametric Bayesian credible
  sets.
\newblock \emph{The Annals of Statistics} 43(4): 1391--1428.

\bibitem[{Umlauft et~al.(2017)Umlauft, Beckers, Kimmel, and
  Hirche}]{umlauftetal_fb_lin_gp}
Umlauft, J.; Beckers, T.; Kimmel, M.; and Hirche, S. 2017.
\newblock Feedback linearization using Gaussian processes.
\newblock In \emph{2017 IEEE 56th Annual Conference on Decision and Control
  (CDC)}, 5249--5255. IEEE.

\bibitem[{Wang, Tuo, and Jeff~Wu(2019)}]{wangtuowu_kriging}
Wang, W.; Tuo, R.; and Jeff~Wu, C. 2019.
\newblock On prediction properties of kriging: Uniform error bounds and
  robustness.
\newblock \emph{Journal of the American Statistical Association} 1--27.

\end{thebibliography}

\clearpage

\appendix

\onecolumn
\section*{Supplementary material}
In this supplementary material, {\color{blue} we first give details on the corrections in this version.}
We then provide the detailed proofs in Section \ref{sect.suppl.theory} that were 
omitted in the main paper, and further details on the numerical examples in 
Section \ref{sect.suppl.theory}.
Before this, we provide a straightforward and illustrative introduction to 
our results and their use in Section \ref{sect.suppl.overview}\footnote{The addition of such a section has been suggested by during the Review process. We thank the reviewers for their very helpful suggestion.}.

{\color{blue}
\section*{Corrections}
This version contains some corrections of the original work, which we now describe. 

Unfortunately, there was a problem with constants in Theorem \ref{thm.main_nominal}. In order to prove this result, we rely on the concentration inequality \cite[Theorem~1]{cg17}, which requires that the nominal noise variance $\lambda$ in the GP regression is at least 1. In order to provide an uncertainty bound for all $\lambda>0$, we introduced a modified nominal variance $\bar\lambda$, but this object was not used correctly in the proof. In this version, we have repaired the issue, cf. Theorem \ref{thm.main_nominal} and its proof. All results that rely on it were changed accordingly (by using the correct formula for $\beta_N$). Among the results in this work, only a constant in Theorem \ref{thm.kernel_misspec} is affected and has also been corrected.
Note that for $\lambda\geq 1$, we have $\bar\lambda=\lambda$, and the results as stated in the original version hold without requiring correction.

Second, the statement of Proposition \ref{prop.se_misspec_simple} was imprecise. This result states that Theorem \ref{thm.main_nominal} holds also in a certain misspecified setting without change. However, since the operator norm of the inclusion map described in Proposition \ref{prop.simple_kernel_misspec} is in general greater than 1, one has to use an RKHS norm that is valid in the misspecified setting. We have added a footnote to clarify this issue. %
}

\section{A user-friendly overview} \label{sect.suppl.overview}
We provide in this section a simple tutorial overview of the approach.
For simplicity, consider a scalar function on $D=[-10,10]$, generated from the RKHS of the Squared Exponential (SE) kernel
with RKHS norm 2. A typical example looks as follows.
\begin{center}
\includegraphics[width=0.5\textwidth]{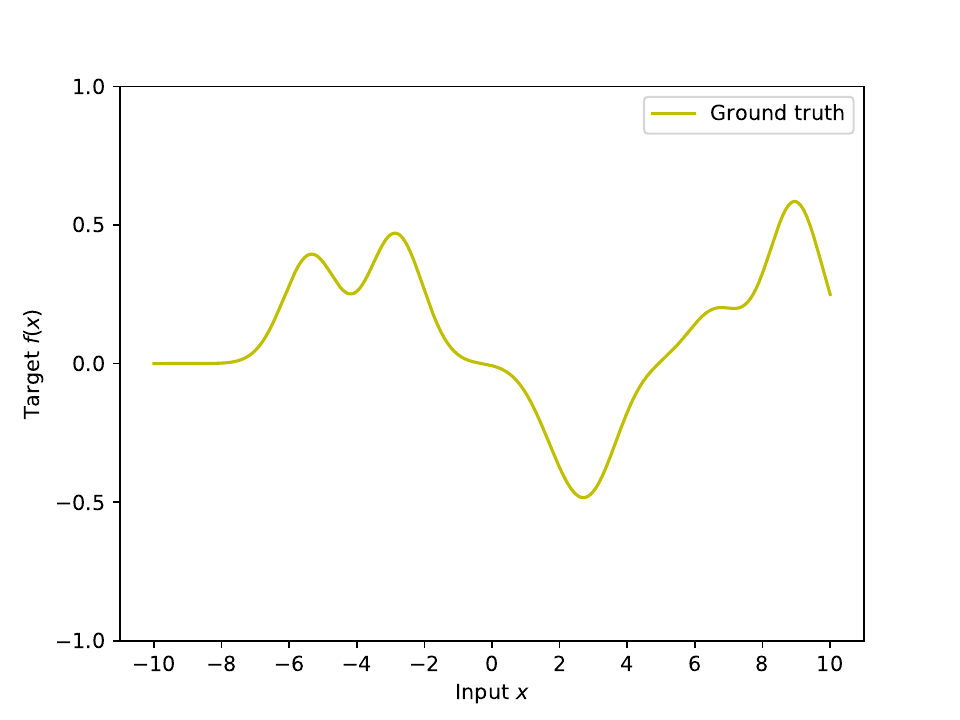}
\end{center}
Next, assume we have a data set of input-output samples from the unknown target function.
Since our bounds are rather flexible in the choice of the underlying concentration results, many common noise settings
are supported. In particular, in Theorem 1 and 5 we use the powerful Theorem 1 from (Chowdhury and Gopalan 2017)
which is compatible with (conditional) subgaussian martingale-difference noise.
For demonstration purposes, we use in this simple example i.i.d. $\mathcal{N}(0,0.1)$ noise, resulting in the following data set
with 100 data points.
\begin{center}
\includegraphics[width=0.5\textwidth]{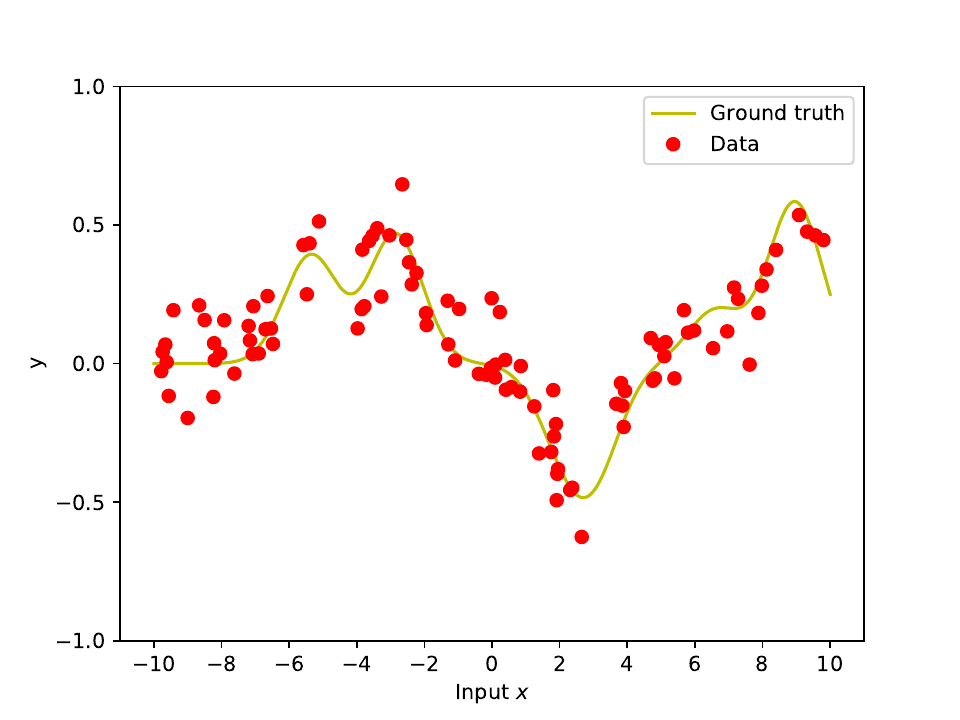}
\end{center}
Since we rely on standard Gaussian Process Regression (GPR), a large variety on modelling tools are available. In particular,
since no additional assumption on the kernel (covariance function) are necessary later on, we can include additional prior knowledge
in a systematic manner in the kernel. For example, linearly constrained constrained GPs (Jidling et al. 2017) or related approaches
like (Geist and Trimpe 2020) could be used.
In this simple example, we do not assume any additional structural and use the SE covariance function with length-scale 0.8
that has been used to generate the target function. However, we like to stress that this is only for demonstrative purposes. Everything would work with multivariate inputs or more general kernels.
Furthermore, for simplicity we use the true noise variance. Again, this is for demonstrative purposes and general subgaussian noise would also work.
This results in the following posterior GP. Note that the uncertainty set here is \emph{not} in a frequentist sense, i.e.,
in general we cannot say with probability at least 2 SD it covers the ground truth. 
\begin{center}
\includegraphics[width=0.5\textwidth]{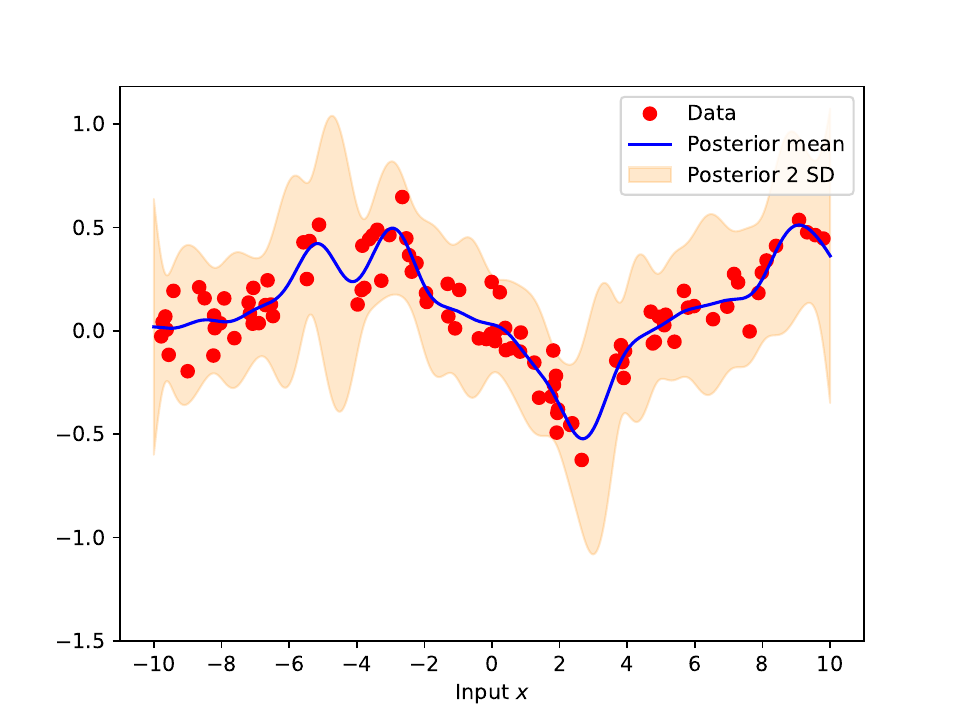}
\end{center}
Suppose now that we need tight frequentist uncertainty sets. This is a relevant scenario in many safe-critical applications,
e.g., in robust learning-based control. First, we need to determine whether the nominal results (e.g., Theorem 1 in the main text)
are applicable or a robust version is necessary. For simplicity, we work in the nominal setting, i.e. we use the SE kernel used to
generate the ground truth. Next, we need a reasonable upper bound on the RKHS norm of the target function,
as in related work like \cite{cg17} and (Maddalena, Scharnhorst, and Jones 2020). Deriving such a bound
from established prior knowledge (in particular, in engineering applications) is ongoing work, so here we simply use
the true RKHS norm and a safety margin like in (Maddalena, Scharnhorst, and Jones 2020) (here we bound the norm by 3).
Note that in contrast to heuristics that directly set a value for the scaling factor $\beta$, the RKHS norm bound has a concrete
interpretation as a measure of complexity of the ground truth and hence is directly related to prior knowledge 
about the target function.
We use $\delta=0.01$ in Theorem 1, resulting in the following uncertainty set.
\begin{center}
\includegraphics[width=0.5\textwidth]{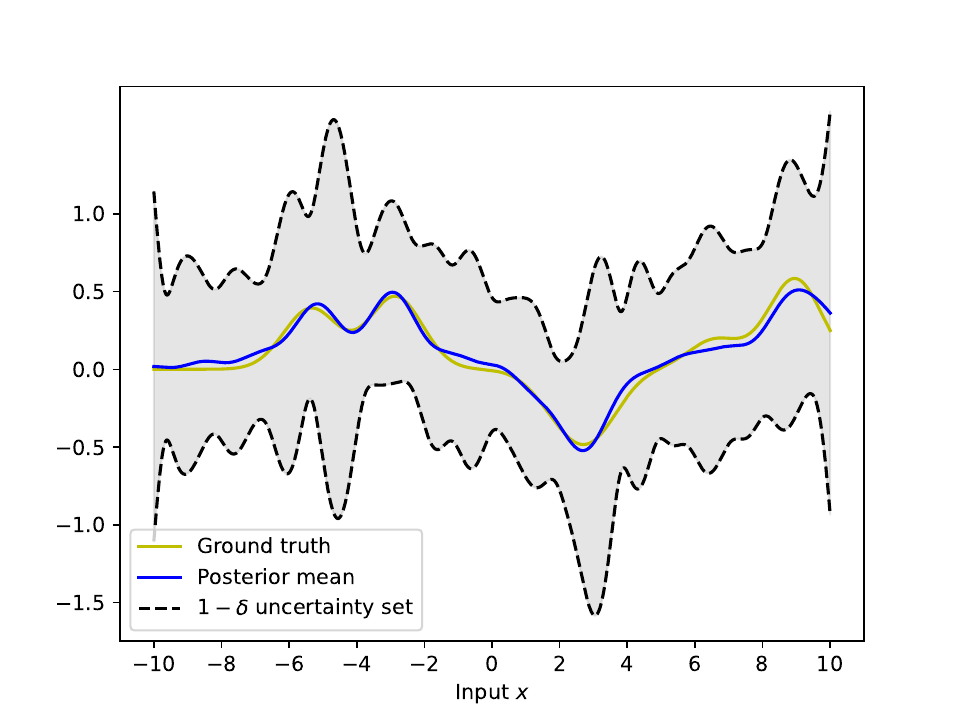}
\end{center}
Note that the uncertainty set covers the ground truth at least with probability $1-\delta=0.99$ w.r.t. the data generating process.
This is exactly the guarantee needed for robust approaches. 
The latter are usually non-stochastic and hence any non-stochastic guarantees derived from the uncertainty set hold with the same high frequentist probability as the uncertainty set itself.
\section{Theory} \label{sect.suppl.theory}

\subsection{Proof of Theorem 1}
Following the proof of Theorem 2 from \cite{cg17} we get for all $x \in D$ that
\begin{align} \label{eq.cg17.1}
| f(x) - \mu_N(x)| & = | f(x) - \bb{k}_N(x)^T (\bb{K}_N + \lambda \Ident_N)^{-1}(\bb{f}_N + \bb{\epsilon}_N) | \\
& \leq | f(x) - \bb{k}_N(x)^T (\bb{K}_N + \lambda \Ident_N)^{-1}\bb{f}_N | 
  + | \bb{k}_N(x)^T (\bb{K}_N + \lambda \Ident_N)^{-1}\bb{\epsilon}_N| \nonumber \\
& \leq \| f \|_k \sigma_N(x) + 
  \frac{1}{\sqrt{\lambda}}\sigma_N(x) \sqrt{\bb{\epsilon}_N^T \bb{K}_N(\bb{K}_N + \lambda \Ident_N)^{-1} \bb{\epsilon}_N}. \nonumber
\end{align}
We furthermore have
\begin{align*}
 \bb{K}_N(\bb{K}_N + \lambda \Ident_N)^{-1} 
 & = \frac{\bar\lambda}{\lambda} \left( \frac{\bar\lambda}{\lambda}\right)^{-1}  \bb{K}_N(\bb{K}_N + \lambda \Ident_N)^{-1} \\
 & = \frac{\bar\lambda}{\lambda}  \bb{K}_N\left(\left( \frac{\bar\lambda}{\lambda}\right) (\bb{K}_N + \lambda \Ident_N)\right)^{-1} \\
 & = \frac{\bar\lambda}{\lambda}  \bb{K}_N\left(\frac{\bar\lambda}{\lambda} \bb{K}_N + \bar\lambda \Ident_N\right)^{-1} \\
 & = \bar{\bb{K}}_N \left(\bar{\bb{K}}_N + \bar\lambda \Ident_N\right)^{-1},
\end{align*}
where we defined $\bar{\bb{K}}_N=\frac{\bar\lambda}{\lambda}  \bb{K}_N$.
Since $\bar{\lambda}\geq 1$ and $\bar{\bb{K}}_N$ is positive definite, we can use Theorem 1 from \cite{cg17} (applied to the kernel $\frac{\bar\lambda}{\lambda} k$ instead of $k$) to get that with probability at least $1-\delta$ 
\begin{equation*}
  \bb{\epsilon}_N^T \bar{\bb{K}}_N(\bar{\bb{K}}_N + \bar{\lambda} \Ident_N)^{-1} \bb{\epsilon}_N
  \leq R\sqrt{\log\left( \determinant(\bar{\bb{K}}_N + \bar{\lambda} \Ident_N) \right) - 2\log(\delta)} \sigma_N(x)
\end{equation*}
for all $x \in D$. Because $\sigma_N$ does not depend on $\bb{\epsilon}_N$, we can use this in \eqref{eq.cg17.1}
and the result follows. \qed

\subsection{Proof of Proposition 2}
 Let $x \in D$ be arbitrary, then
 \begin{align*}
  | f(x) - \mu_N(x)| & = | f(x) - \bb{k}_N(x)^T (\bb{K}_N + \lambda \Ident_N)^{-1}(\bb{f}_N + \bb{\epsilon}_N) | \\
  & \leq | f(x) - \bb{k}_N(x)^T (\bb{K}_N + \lambda \Ident_N)^{-1}\bb{f}_N | 
    + | \bb{k}_N(x)^T (\bb{K}_N + \lambda \Ident_N)^{-1}\bb{\epsilon}_N|.
 \end{align*}
 Exactly as in the proof of Theorem 2 in \cite{cg17} we have
 \begin{equation*}
   | f(x) - \bb{k}_N(x)^T (\bb{K}_N + \lambda \Ident_N)^{-1}\bb{f}_N | \leq B\sigma_N(x)
 \end{equation*}
 and since (use Cauchy-Schwarz)
 \begin{align*}
  | \bb{k}_N(x)^T (\bb{K}_N + \lambda \Ident_N)^{-1}\bb{\epsilon}_N| 
  & \leq \| (\bb{K}_N + \lambda \Ident_N)^{-1} \bb{k}_N(x) \| \| \bb{\epsilon}_N \|
  =  \| (\bb{K}_N + \lambda \Ident_N)^{-1} \bb{k}_N(x) \| \sqrt{\bb{\epsilon}_N^T \Ident_N \bb{\epsilon}_N}
 \end{align*}
 we get from Theorem 2.1 in (Hsu et al. 2021) that
 \begin{equation*}
   \Pp\pbl \|\bb{\epsilon}_N \|^2 \leq  R^2 (N + 2\sqrt{N}\sqrt{\log(\frac{1}{\delta})} + 2\log(\frac{1}{\delta})) \pbr \geq 1 - \delta
 \end{equation*}
 and the result follows. \qed
 
\subsection{Proof of Theorem 5}
Denote by $\tilde{\mu}_N$, $\tilde{\sigma}_N^2$, $\tilde{\bb{K}}_N$ the posterior mean, posterior variance
and Gram matrix of the GP, but with $\tilde{k}$ as covariance function, and analogously $\tilde{\bb{k}}_N$.
Let $x \in D$ be arbitrary, then we have
\begin{equation*}
 |f(x) - \mu_N(x)| \leq | f(x) - \tilde{\mu}_N(x) | + | \tilde{\mu}_N(x) - \mu_N(x) |.
\end{equation*}
Our strategy will be to bound the first term on the right-hand-side with Theorem 1
and then upper bound all resulting or remaining quantities by expressions involving only $k$ instead of $\tilde{k}$.
As a preparation we first derive some elementary bounds that are frequently used later on.
By assumption,
\begin{equation} \label{eq.term1}
 \| \tilde{\bb{k}}_N(x) - \bb{k}_N(x) \| = \sqrt{\sum_{i=1}^N (\tilde{k}(x,x_i) - k(x,x_i))^2 } \leq \sqrt{N} \tilde{\epsilon}
\end{equation}
and hence (using the triangle inequality)
\begin{equation} \label{eq.term2}
 \| \tilde{\bb{k}}_N(x) \| \leq \| \bb{k}_N(x) \| + \| \tilde{\bb{k}}_N(x) - \bb{k}_N(x) \| \leq \| \bb{k}_N(x) \| + \sqrt{N} \tilde{\epsilon}.
\end{equation}
Furthermore, since $\tilde{\bb{K}}_N$ is positive semidefinite
\begin{equation*}
  \| (\tilde{\bb{K}}_N + \lambda \Ident_N )^{-1} \| = \lambda_{\text{max}}( (\tilde{\bb{K}}_N + \lambda \Ident_N )^{-1})
  = \frac{1}{\lambda_{\text{min}}(\tilde{\bb{K}}_N + \lambda \Ident_N )} \leq \frac{1}{\lambda}
\end{equation*}
and hence together with the triangle inequality
\begin{equation} \label{eq.term3}
 \| (\tilde{\bb{K}}_N + \lambda \Ident_N )^{-1} - (\bb{K}_N + \lambda \Ident_N)^{-1} \| 
 \leq \frac{1}{\lambda} + \|  (\bb{K}_N + \lambda \Ident_N)^{-1} \|. 
\end{equation}
Finally, using first the triangle inequality and then the submultiplicativity of the spectral norm we get
\begin{align*}
  & \| (\tilde{\bb{K}}_N + \lambda \Ident_N )^{-1}\tilde{\bb{k}}_N(x) - (\bb{K}_N + \lambda \Ident_N)^{-1}\bb{k}_N(x) \| \\
  & \hspace{1cm} \leq \| \left( \tilde{\bb{K}}_N + \lambda \Ident_N )^{-1} -  (\bb{K}_N + \lambda \Ident_N)^{-1} \right)\tilde{\bb{k}}_N(x) \| 
      + \| (\bb{K}_N + \lambda \Ident_N)^{-1} ( \tilde{\bb{k}}_N(x) - \bb{k}_N(x) ) \| \\
  & \hspace{1cm} \leq \| (\tilde{\bb{K}}_N + \lambda \Ident_N )^{-1} -  (\bb{K}_N + \lambda \Ident_N)^{-1}  \| \|\tilde{\bb{k}}_N(x) \| 
      + \| (\bb{K}_N + \lambda \Ident_N)^{-1} \| \|  \tilde{\bb{k}}_N(x) - \bb{k}_N(x) \|
\end{align*}
and hence from \eqref{eq.term1}, \eqref{eq.term2}, \eqref{eq.term3}
\begin{align} \label{eq.term4}
 \| (\tilde{\bb{K}}_N + \lambda \Ident_N )^{-1}\tilde{\bb{k}}_N(x) - (\bb{K}_N + \lambda \Ident_N)^{-1}\bb{k}_N(x) \|
 \leq C_N(x)
\end{align}
with
\begin{equation*}
 C_N(x) = \left( \frac{1}{\lambda} + \|  (\bb{K}_N + \lambda \Ident_N)^{-1} \| \right)(\| \bb{k}_N(x) \| + \sqrt{N} \tilde{\epsilon})
    + \| (\bb{K}_N + \lambda \Ident_N)^{-1} \| \sqrt{N} \tilde{\epsilon}
\end{equation*}
Now,
\begin{align*}
 | \tilde{\mu}_N(x) - \mu_N(x) | 
 & = | \tilde{\bb{k}}_N(x)  (\tilde{\bb{K}}_N + \lambda \Ident_N )^{-1} \bb{y}_N - \bb{k}_N(x)  (\bb{K}_N + \lambda \Ident_N)^{-1} \bb{y}_N | \\
 & \leq  \| (\tilde{\bb{K}}_N + \lambda \Ident_N )^{-1}\tilde{\bb{k}}_N(x) + (\bb{K}_N + \lambda \Ident_N)^{-1}\bb{k}_N(x)\|\|\bb{y}_N\| \\
 & \leq C_N(x) \| \bb{y}_N \|,
\end{align*}
where we used Cauchy-Schwarz in the first inequality and \eqref{eq.term4} in the second.
Using Theorem 1 we get that
 \begin{equation*}
 \Pp\pbl | \tilde{\mu}_N(x) - f(x) | \leq \tilde{\beta}_N \tilde{\sigma}_N(x) \: \forall N \in \N, x \in D \pbr \geq 1 - \delta
\end{equation*}
where
 \begin{equation*}
 \tilde{\beta}_N = B + \frac{R}{\sqrt{\lambda}}\sqrt{\log\left( \determinant(\bar{\lambda}/\lambda\tilde{\bb{K}}_N + \bar\lambda \Ident_N) \right) - 2\log(\delta)}.
\end{equation*}
Let $\lambda_i(\tilde{\bb{K}}_N)$ be the $i$-th largest eigenvalue of $\tilde{\bb{K}}_N$, then we get from
Weyl's inequality and the definition of the Frobenius norm that
\begin{equation*}
 \lambda_i(\tilde{\bb{K}}_N) \leq \lambda_i(\bb{K}_N) + \| \tilde{\bb{K}}_N - \bb{K}_N \| \leq \lambda_i(\bb{K}_N) + N \tilde{\epsilon},
\end{equation*}
and hence
\begin{align*}
 \log \determinant \left(\frac{\bar{\lambda}}{\lambda}\tilde{\bb{K}}_N + \bar\lambda \Ident_N \right) 
  & = \log \left( \prod_{i=1}^N \lambda_i\left(\frac{\bar{\lambda}}{\lambda}\tilde{\bb{K}}_N + \bar\lambda \Ident_N \right) \right) \\
 & = \log \left( \prod_{i=1}^N \left(\frac{\bar{\lambda}}{\lambda}\lambda_i(\tilde{\bb{K}}_N)  + \bar\lambda\right) \right) \\
 & = \sum_{i=1}^N \log\left(\frac{\bar{\lambda}}{\lambda}\lambda_i(\tilde{\bb{K}}_N) + \bar\lambda\right) \\
 & \leq \sum_{i=1}^N \log\left(\frac{\bar{\lambda}}{\lambda}\lambda_i(\bb{K}_N) 
  + \frac{\bar{\lambda}}{\lambda} N \tilde{\epsilon} + \bar\lambda\right) \\
 & = \log \determinant \left(\frac{\bar{\lambda}}{\lambda} \bb{K}_N + \left(\frac{\bar{\lambda}}{\lambda}N\tilde{\epsilon} + \bar\lambda\right) \Ident_N \right).
\end{align*}
In particular,
\begin{equation*}
  \tilde{\beta}_N = B + \frac{R}{\sqrt{\lambda}}\sqrt{\log\left( \determinant(\bar{\lambda}/\lambda\tilde{\bb{K}}_N + \bar\lambda \Ident_N) \right) - 2\log(\delta)}
  \leq  B +  \frac{R}{\sqrt{\lambda}}\sqrt{\log \determinant \left(\frac{\bar{\lambda}}{\lambda} \bb{K}_N + \left(\frac{\bar{\lambda}}{\lambda}N\tilde{\epsilon} + \bar\lambda\right) \Ident_N \right) - 2\log(\delta)} =: \bar{\beta}_N
\end{equation*}
Turning to the posterior variance, we get from the triangle inequality
\begin{align*}
 \tilde{\sigma}_N^2(x) & \leq \sigma_N^2(x) + | \sigma_N^2(x) - \tilde{\sigma}_N^2(x) |.
\end{align*}
We continue with 
\begin{align*}
 | \sigma_N^2(x) - \tilde{\sigma}_N^2(x) | & = | k(x,x) - \bb{k}_N(x)^T ( \bb{K}_N + \lambda \Ident_N )^{-1} \bb{k}_N(x)
    - \tilde{k}(x,x) + \tilde{\bb{k}}_N(x)^T ( \tilde{\bb{K}}_N + \lambda \Ident_N )^{-1} \tilde{\bb{k}}_N(x)| \\
 & \leq | k(x,x) - \tilde{k}(x,x)| + | (\tilde{\bb{k}}_N(x) - \bb{k}_N(x) )^T( \tilde{\bb{K}}_N + \lambda \Ident_N )^{-1} \tilde{\bb{k}}_N(x)| \\
    & \hspace{1cm} + |\bb{k}_N(x)^T( ( \tilde{\bb{K}}_N + \lambda \Ident_N )^{-1} \tilde{\bb{k}}_N(x) - 
    (\bb{K}_N + \lambda \Ident_N )^{-1} \bb{k}_N(x)) | \\
 & \leq | k(x,x) - \tilde{k}(x,x)| + \| \tilde{\bb{k}}_N(x) - \bb{k}_N(x) \|\| ( \tilde{\bb{K}}_N + \lambda \Ident_N )^{-1}\tilde{\bb{k}}_N(x)\|\\
    & \hspace{1cm} + \| \bb{k}_N(x) \| \| ( \tilde{\bb{K}}_N + \lambda \Ident_N )^{-1} \tilde{\bb{k}}_N(x) 
	- (\bb{K}_N + \lambda \Ident_N )^{-1} \bb{k}_N(x)\| \\
 & \leq \tilde{\epsilon} 
    + \sqrt{N} \tilde{\epsilon}  \| (\bb{K}_N + \lambda \Ident_N )^{-1}\bb{k}_N(x) \|  
    +  (\sqrt{N} \tilde{\epsilon} +  \| \bb{k}_N(x) \|) C_N(x) = S^2_N(x),
\end{align*}
where we used the triangle inequality again in the first inequality, Cauchy-Schwarz in the second inequality
and finally \eqref{eq.term1}, \eqref{eq.term2}, \eqref{eq.term3} together with
\begin{align*}
 \| ( \tilde{\bb{K}}_N + \lambda \Ident_N )^{-1}\tilde{\bb{k}}_N(x)\|
  & \leq \| (\bb{K}_N + \lambda \Ident_N )^{-1}\bb{k}_N(x) \|
    + \| ( \tilde{\bb{K}}_N + \lambda \Ident_N )^{-1}\tilde{\bb{k}}_N(x) - (\bb{K}_N + \lambda \Ident_N )^{-1}\bb{k}_N(x) \| \\
  & \leq  \| (\bb{K}_N + \lambda \Ident_N )^{-1}\bb{k}_N(x) \| 
    + C_N(x)
\end{align*}
Putting everything together, we find that with probability at least $1-\delta$
\begin{equation*}
 | \mu_N(x) - f(x) | \leq C_N(x)\|\bb{y}_N\| + \tilde{\beta}_N \tilde{\sigma}_N(x)
\end{equation*}
and therefore, using the upper bounds on $\tilde{\beta}_N$ and $\tilde{\sigma}_N(x)$ derived above, that
with probability at least $1-\delta$
\begin{equation*}
 | \mu_N(x) - f(x) | \leq B_N(x)
\end{equation*}
where
\begin{equation*}
 B_N(x) =  C_N(x)\|\bb{y}_N\| + \bar{\beta}_N \sqrt{\sigma_N^2(x) + \tilde{\epsilon} 
    + \sqrt{N} \tilde{\epsilon}  \| (\bb{K}_N + \lambda \Ident_N )^{-1}\bb{k}_N(x) \|  
    +  (\sqrt{N} \tilde{\epsilon} +  \| \bb{k}_N(x) \|) C_N(x)}
\end{equation*} \qed

\subsection{An alternative robustness result}
The proof of Theorem 5 can be easily adapted to other nominal bounds. In order to illustrate this,
we now state and prove a robust version of Proposition 2.
\begin{proposition*} \label{prop.kernel_misspec_independent_setting}
 Consider the situation of Proposition 2, but this time assume that the target function $f$ is from 
 the RKHS $(\tilde{H}, \|\cdot\|_{\tilde{k}})$ of a different kernel $\tilde{k}$ such that
 still $\|f\|_{\tilde{k}} \leq B$ and $ \sup_{x,x^\prime \in D} |k(x,x^\prime)-\tilde{k}(x,x^\prime)| \leq \tilde{\epsilon}$
 for some $\tilde{\epsilon}\geq 0$. We then have for any $\delta \in (0,1)$ with
 \begin{equation*}
 \tilde{\eta} = R(\| (\bb{K}_N + \lambda \Ident_N)^{-1} \bb{k}_N(x) \| + C_N(x))
 \sqrt{N + 2\sqrt{N}\sqrt{\log(\frac{1}{\delta})} + 2\log(\frac{1}{\delta})}
\end{equation*}
 that
 \begin{equation*}
 \Pp\pbl | \mu_N(x) - f(x) | \leq B\sqrt{\sigma^2_N(x) + S_N^2(x)} + C_N(x)\|\bb{y}_N\|+ \tilde{\eta}_N(x) \: \forall x \in D \pbr \geq 1 - \delta
\end{equation*}
 where $C_N(x)$ and $S^2_N(x)$ are defined in the main text by (4) and (5), respectively.
\end{proposition*}
\begin{proof}
Let $x \in D$ be arbitrary, then we have
\begin{equation*}
 |f(x) - \mu_N(x)| \leq | f(x) - \tilde{\mu}_N(x) | + | \tilde{\mu}_N(x) - \mu_N(x) |.
\end{equation*}
The second term can be upper bounded by $C_N(x) \| \bb{y}_N \|$ as in the proof of Theorem 5.
Using Proposition 2 we have that with probability at least $1-\delta$ for all $x \in D$
\begin{equation*}
| \tilde{\mu}_N(x) - f(x) | \leq B\tilde{\sigma}_N(x) +  R \| (\tilde{\bb{K}}_N + \lambda \Ident_N)^{-1} \tilde{\bb{k}}_N(x) \| 
 \sqrt{N + 2\sqrt{N}\sqrt{\log(\frac{1}{\delta})} + 2\log(\frac{1}{\delta})}
\end{equation*}
From the proof of Theorem 5 we have
\begin{equation*}
 | \sigma_N^2(x) - \tilde{\sigma}_N^2(x) | \leq S^2_N(x)
\end{equation*}
and
\begin{equation*}
  \| (\tilde{\bb{K}}_N + \lambda \Ident_N)^{-1} \tilde{\bb{k}}_N(x) \| C_N(x)
\end{equation*}
and the result follows.
\end{proof}

\section{Numerical experiments} \label{sect.suppl.experiments}
We now provide more details on the numerical experiments as well as additional remarks and results.

\subsection{Generating functions from an RKHS} \label{sect.suppl.rkhs_sampling}
For the numerical experiments we need to generate ground truths, i.e. we need to randomly generate
functions belonging to the RKHS of a given kernel. A generic approach is to use the pre-RKHS of the
kernel which is contained (even densely w.r.t. the kernel norm) in the actual RKHS, cf. \cite[Theorem 4.21]{svm_book}
for details. Let $X$ be a set and $k$ a kernel on $X$. For any $N \in \N$, $x_1,\ldots,x_N \in X$
and $\bb{\alpha} \in \R^N$ the function defined by
\begin{equation*}
 x \mapsto \sum_{n=1}^N \alpha_n k(x_n, x)
\end{equation*}
is contained in the (unique) RKHS corresponding to $k$ and has RKHS norm
$\sqrt{\bb{\alpha}^T \bb{K} \bb{\alpha}}$, where $\bb{K}=(k(x_i,x_j))_{i,j=1,\ldots,N}$ is the corresponding Gram matrix.
It is hence possible to generate an RKHS function $f$ of prescribed RKHS norm $B$ by
randomly sampling inputs $x_1,\ldots,x_N \in X$ and coefficients $\tilde{\alpha} \in \R^N$
and setting
\begin{equation} \label{eq.sample_rkhs}
 f(x) = \sum_{n=1}^N \alpha_n k(x_n, x)
\end{equation}
where $\bb{\alpha}=\frac{B}{\sqrt{\bb{\alpha}^T \bb{K} \bb{\alpha}}}\tilde{\bb{\alpha}}$. Of course, $f$ can only be evaluated
at finitely many points $\tilde{X} \subseteq X$.

More concretely, we fix a finite evaluation grid $\tilde{X} \subseteq X$, choose uniformely a number 
$N \in [N_{\text{min}}, N_{\text{max}}] \cap \N$, choose uniformely $N$ pairwise different points $x_1,\ldots,x_N \in \tilde{X}$,
sample $\tilde{\alpha}_i \distr \Norm(0, \sigma_f^2)$ and apply the construction \eqref{eq.sample_rkhs}. For precise choices
of the parameters are given below.

We would like to point out an important aspect. This article is concerned with frequentist results, i.e.,
there is a ground truth from a collection of possible ground truths and the results have to hold for each of these possible
ground truths. In particular, even if a ground truth might be considered pathological, the results have to hold if they are to
be considered \emph{rigorous}. This aspect is important for numerical experiments, especially when trying to assess the
conservatism of a result. In our setting it might happen that the results seem very conservative for functions that are randomly
generated in a certain fashion, but there are RKHS functions (which might be difficult to generate) for which the results
might be sharp. 
Let us illustrate this point with the Gaussian kernel. We use a uniform grid of 1000 points from $[-1,1]$
together with the Gaussian kernel with length scale 0.2. For the pre-RKHS approach we use $N_{\text{min}}=5$
and $N_{\text{max}}=200$ and $\sigma_f^2=1$. As an alternative, we use the ONB described in \cite[Section~4.4]{svm_book} and
consider only the first 50 basis functions from \cite[Equation~(4.35)]{svm_book} for numerical reasons.
We first select the number of basis functions $N$ to use uniformely between 5 and 50 and then choose 
$N$ such functions uniformely. As coefficients we sample $\alpha_i \distr \Norm(0, 1)$ i.i.d. for $i=1,\ldots,N$
and normalize (w.r.t. to $\ell_2$-norm) and multiply by the targeted RKHS norm. For both the pre-RKHS approach and the
ONB approach we use $\|\cdot\|_k=2$ and sample 4 functions each. 
The result is shown in Figure \ref{fig.illustrating_sampling_rkhs}.
\begin{figure}
 \includegraphics[width=\textwidth]{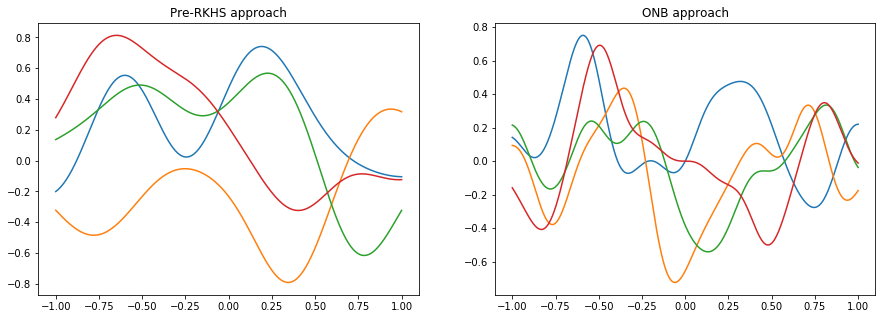}
  \caption{Illustrating sampling from the Gaussian kernel with the pre-RKHS method (left) and with an explicit ONB (right).
 Details are provided in the text.}
  \label{fig.illustrating_sampling_rkhs}
\end{figure}
Clearly, the resulting functions have a different shape, despite having the same RKHS norm with respect to the same kernel.
In particular, the functions generated using the ONB approach seem to make sharper turns.

We like to stress that this strongly suggests that in a frequentist setting one has to be careful with statements
about the conservatism of a proposed bound or method that are based purely on empirical observations.
It might be that the method for generating ground truths has a certain bias, i.e. has a tendency to produce
only ground truths from a certain region of the space of all ground truths.

\subsection{Details on experiments with synthetic data} \label{sect.suppl.synth_experiments}
Unless otherwise stated, we use $[-1,1]$ as the input set and consider a uniform grid of 1000 points for function evaluations.
For the pre-RKHS approach we use $N_{\text{min}}=5$ and $N_{\text{max}}=200$ and $\sigma_f^2=1$ in all experiments.

Unless otherwise stated, in each experiment we sample 50 RKHS functions as ground truth. For each ground truth we
generate 10000 training sets by randomly sampling 50 input points uniformly from the 1000 evaluation points
and add i.i.d. zero-mean normal noise with SD 0.5. For each training set we run Gaussian Process Regression (which we call a learning instance)
and determine the uncertainty set for a specific setting (evaluated again at the 1000 evaluation points).
We consider a learning instance a failure if the uncertainty set does not fully cover the ground truth at all 1000 evaluation points.

For convenience, each experiment has a tag with prefix \emph{exp\_}.
\subsubsection{Testing the nominal bound}
Here we use the SE kernel with length scale 0.2 (\emph{exp\_1\_1\_a})
and the Matern kernel with length scale 0.2 and $\nu=1.5$ (\emph{exp\_1\_1\_b}).
We generate RKHS functions of RKHS norm 2 using the pre-RKHS approach and use the same kernel for generating
the ground truth and running GPR. The nominal noise level of GPR is set to $\lambda=0.5$.
The uncertainty set is generated using Theorem 1 with $B=2$, $R=0.5$ and $\delta=0.1, 0.01, 0.001, 0.0001$.
As already reported in the main text, a violation of the uncertainty set was found in no instance 
(i.e. for all 50 RKHS functions and each of the 10000 learning instances). The mean of the scalings $\beta_{50}$
(together with 1 SD, average is over all
50 RKHS functions and all learning instances) is shown in Table 1 in the main text.
\subsubsection{Exploring conservatism}
In order to explore the potential conservatism of Theorem 1 we repeated the previous experiments
with $\delta=0.01$ and replaced $\beta_{50}$ by 20 equidistant scalings between 2 and $\beta_{50}$.
We used this changed setup for the SE kernel (\emph{exp\_1\_2\_a}), Matern kernel (\emph{exp\_1\_2\_b})
and SE kernel together with the ONB sampling approach (\emph{exp\_1\_2\_c}).
Whereas for the Matern kernel also the heuristic $\beta=2$ works for this example (still no uncertainty set violations),
the situation is rather different for the SE kernel. As shown in Figure \ref{fig.heuristic_se_failure}
for $\beta$ close to 2 the frequency of uncertainty violations is much higher than 0.01, in particular
for the case of sampling from the ONB.
\begin{figure}
 \includegraphics[width=0.5\textwidth]{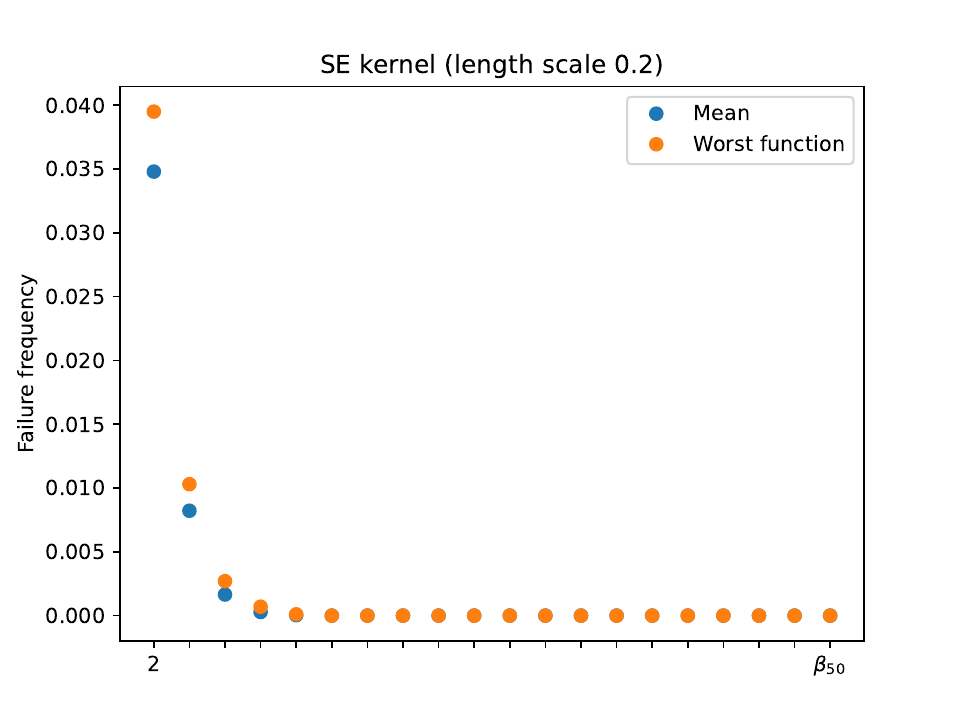}
 \includegraphics[width=0.5\textwidth]{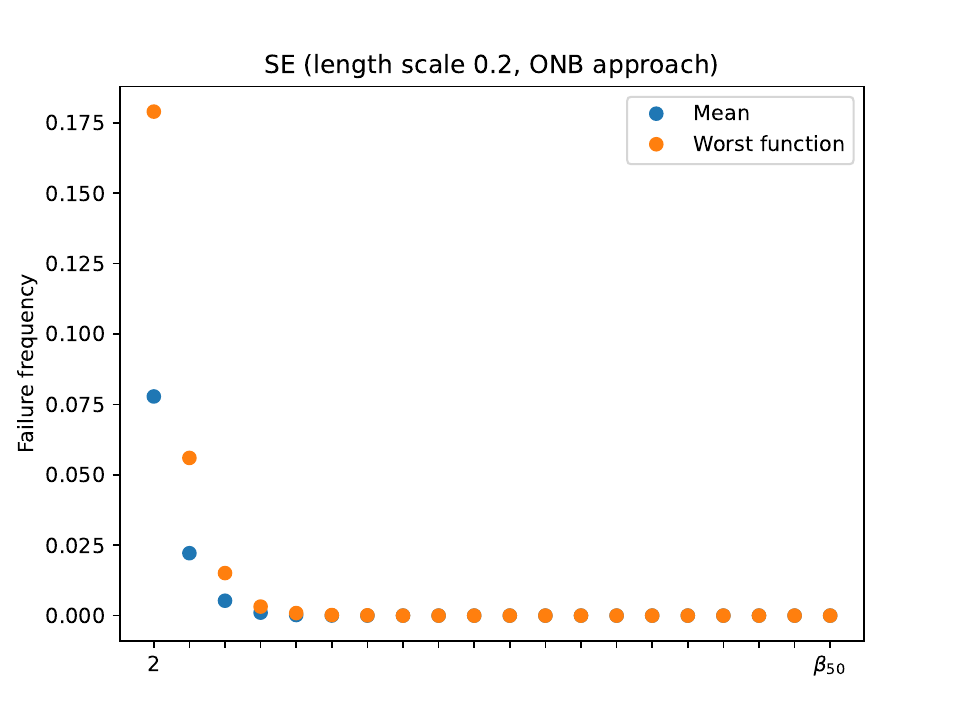}
 \caption{Exploring conservatism of Theorem 1. For each target function and learning instance 20 different uncertainty sets are tested.
 Each such uncertainty is generated from Theorem 1 by replacing $\beta_{50}$
 by $\beta=2,\ldots,\beta_{50}$ (equidistant). Worst function means the highest failure frequency among all 50 target functions
 for the particular scaling. Ground truths sampled with pre-RKHS approach (left) and ONB approach (right).}
 \label{fig.heuristic_se_failure}
\end{figure}

\subsubsection{Misspecified kernel, benign setting} 
We now consider model misspecification. We repeat the first experiment \emph{exp\_1\_1\_a},
but now generate RKHS functions from the SE kernel with lengthscale 0.5 
and use the SE kernel with lengthscale 0.2 in GPR (\emph{exp\_1\_3\_a}). As discussed in the main text,
Proposition 4 indicates that this is a benign form of model misspecification.
Indeed, similar to the case of the correct kernel we find no uncertainty set violation and similar scalings.
The corresponding scalings $\beta_{50}$ are reported in Table 2 (upper row)
in the main text (only for ONB sampling approach, no significant difference compared to pre-RKHS approach).
\subsubsection{Misspecified kernel, problematic setting}
We repeat the previous experiment \emph{exp\_1\_3\_a},
but now generate RKHS functions from the SE kernel with lengthscale 0.2
and use the SE kernel with lengthscale 0.5 in GPR. We use both the pre-RKHS approach
(\emph{exp\_1\_4\_a}) and the ONB approach (\emph{exp\_1\_4\_b}).
The resulting scalings are reported in Table 2 (lower row) in the main text,
again only for the ONB sampling approach.
Interestingly, when using the pre-RKHS sampling approach no uncertainty set violation could be found,
but for the ONB sampling we found 
some target function for which the uncertainty violation was higher than the prescribed $\delta$.

\subsubsection{Robust result for misspecified setting}
Finally, we repeat Experiment \emph{exp\_1\_4\_b} from the previous paragraph, but now using Theorem 5 instead of
Theorem 1. We find no violation of the uncertainty set (over all 50 functions tested, all 10000 learning instances for each function
and all $\delta$ tested). Since now the width of the uncertainty set is not a constant rescaling of the posterior standard deviation anymore,
we report the mean (over all 50 functions and each of the 10000 learning instances) of the average width (over the input space) of the uncertainty
sets ($\pm$ SD w.r.t. averaging over all 50 functions and each of the 10000 learning instances) and the SD
(w.r.t. to averaging over the input space), $\pm$ SD w.r.t. averaging over all 50 functions and each of the 10000 learning instances,
in Table 3 in the main text.

\subsubsection{Reproducibility and computational complexity}
All numerical experiments in this section were implemented with Python 3.8.3 (together with Numpy 1.18.1) and
run on a Intel(R) Xeon(R) CPU E5-1680 v4 with 3.40GHz and 78 GiB memory 
(using Ubuntu 18.04.3 LTS). We used the joblib library (version 0.15.1) with $n\_jobs=14$
and used the Gaussian Process Regression implementation from scikit-learn, version 0.22.1.
Each experiment took less than 5 minutes and required less than 1.5GiB memory (monitored using htop).
Note that all experiments can easily be up and down scaled, depending on the available hardware.
The code used for the experiments and figures can be found at \url{https://github.com/Data-Science-in-Mechanical-Engineering/UncertaintyBounds21}.

\subsection{Control example} \label{sect.suppl.control}
We now provide more details on the control example from Section 4.2. 
\subsubsection{Background} \label{sect.suppl.control_background}
For convenience we now provide a cursory overview of background material from control theory.
We can only provide a sketch and refer to standard textbooks for more details, e.g. \cite{astrom_murray} 
for a general introduction to control and \cite{RMD} for a comprehensive introduction
to MPC.

A common goal in control is feedback stabilization under state and input constraints.
Consider a discrete-time dynamical system (or control system) described by
\begin{equation*}
x_+ = f(x,u)
\end{equation*}
with state space $X$, input space $U$ and transition function $f: X \times U \rightarrow X$.
For simplicity assume that $X=\mathbb{R}^n$, $U=\mathbb{R}^m$ and that $f(0,0)=0$, i.e.,
$(0,0)$ is an equilibrium. Furthermore, consider state constraints $\mathbb{X}\subseteq X$ and
input constraints $\mathbb{U}\subseteq U$. Feedback stabilization amounts now to finding a map
$\mu: X \rightarrow U$ such that $x_\ast=0$ is an asymptotically stable equilibrium for the resulting closed loop system described by
\begin{equation*}
x_+ = f(x, \mu(x)),
\end{equation*}
and all resulting state-input trajectories are contained in the constraint set $\mathbb{X}\times\mathbb{U}$.
Note that this requires restriction of the set of possible initial values.

In many applications not only stability, but also a form of optimality is required from the control system.
For example, assume that being in state $x$ and applying input $u$ incurs a cost of $\ell(x,u)$.
If the control system is run for a long time, then we would like a feedback $\mu$ that not only stabilizes the system, but also incurs a small infinite horizon cost
\begin{equation*}
\sum_{n=0}^\infty \ell(x(n), \mu(x(n))),
\end{equation*}
where $x(n)$ is the resulting state trajectory.

One common methodology for dealing with state constraints and optimal control tasks is MPC.
If the system is in state $x$, MPC solves a finite horizon open loop problem, i.e. it determines
a sequence $u(0),\ldots,u(N-1)$ of admissible input values that minimize some cost criterion.
Only the first input $u(0)$ is applied to the system and this process is repeated at the next time instance.
There is a comprehensive theory available on how to design the open loop optimal control problem solved in
each instance, in order to achieve desired closed loop properties. For details we refer to Chapters 1 and 2 in 
\cite{RMD}.

In many applications a control system has to deal with disturbances. Frequently the disturbances can be modelled in an additive manner, i.e. we have a control system of the form
\begin{equation*}
x_+ = f(x,u) + w,
\end{equation*}
where $w \in \mathbb{W} \subseteq \mathbb{R}^n$ is an external disturbance. The feedback stabilization problem
under constraints can now be adapted to this setting, resulting in robust feedback stabilization under constraints.
The goal is now to find a feedback that ensures constraint satisfaction and stabilizes the origin in a relaxed sense (which depends
on the size of the disturbance set $\mathbb{W}$). Furthermore, even in this more challenging situation one might have to deal
with additional cost criterions.

MPC can be adapted to the setting with disturbances. The key idea of most approaches is to solve a constrained 
open loop optimal control problem where the constraints are tightened. The intuition is that even the worst case disturbance cannot
throw the system out of the allowed state-input set. It is clear that this requires sufficiently small bounds on the size
of the disturbances. For more details we refer to Chapter 3 in \cite{RMD}.

\subsubsection{Details on the example} \label{sect.suppl.control_example}
The control example in the main text is from Section 6 in \cite{solopertoetal_learning_rmpc_gp}. It consists of the following system
\begin{equation} \label{eq.system} \small
	\begin{bmatrix}
		x_1^+ \\
		x_2^+
	\end{bmatrix}
	=
	\begin{bmatrix}
		0.995 & 0.095\\
		-0.095 & 0.900
	\end{bmatrix}
	\begin{bmatrix}
		x_1 \\
		x_2
	\end{bmatrix}
	+
	\begin{bmatrix}
		0.048 \\
		0.95
	\end{bmatrix}
	u
	+
	\begin{bmatrix}
		0 \\
		-r(x_2)
	\end{bmatrix}
\end{equation}
modelling a mass-spring-damper system with some nonlinearity $r$ (this could be interpreted as a friction term). 
As described in the main text, we replaced the Stribeck friction curve used by \cite{solopertoetal_learning_rmpc_gp} with a synthetic nonlinearity generated from a known RKHS. Furthermore, the nonlinearity is assumed to be unknown and has to be learned from data. 
The control goal is the stabilization of the origin subject to the state and control constraints $\mathbb{X}=[-10,10] \times [-10,10]$ and $\mathbb{U}=[-3,3]$, as well as minimizing a quadratic cost $\ell(x,u)=10\|x\| + \|u\|$.

The RMPC approach from\cite{solopertoetal_learning_rmpc_gp}  performs this task by interpreting \eqref{eq.system} 
as a linear system with disturbance, given by the nonlinearity $r$, whose graph is a-priori known to lie in the
set $\mathbb{W}_0 = [-10,10] \times [-7,7]$. 
The RMPC algorithm requires as an input disturbance sets $\mathbb{W}(x)$ such that $\begin{pmatrix}0 &-r(x_2)\end{pmatrix}^\top \in \mathbb{W}(x)$ for all $x \in \mathbb{X}$,
which are in turn used to generate tightened nominal constraints ensuring robust constraint satisfaction.
Furthermore, the tighter the sets $\mathbb{W}(x)$ are, the better is the performance of the algorithm, 
cf. Chapter 3 in \cite{RMD} for an in-depth discussion.

We now describe the learning part of this example in more detail: The nonlinearity $r$ (which will be our ground truth)
is sampled from the RKHS of the SE kernel
\begin{equation*}
k(x,x^\prime) = 4\exp \left( - \frac{(x-x^\prime)^2}{2 \times 0.8^2} \right)
\end{equation*}
with RKHS norm 2.
Following \cite{solopertoetal_learning_rmpc_gp}, we uniformly sample
100 partial states $x_2 \in [-10,10]$, evaluate $r$ at these and add i.i.d. Gaussian noise with a standard deviation of 0.01 to it.
The unknown function is then learned using GPR (using the nominal setting, i.e. with known $k$)
from this data set. Theorem 1 then leads to an uncertainty set of the form
$\mathbb{W}(x)=[\mu_{100}(x_2) - \beta_{100} \sigma_{100}(x_2), \mu_{100}(x_2) + \beta_{100} \sigma_{100}(x_2)],$
where we use $\delta:=0.001$. 
In particular, with probability at least $1-\delta$ we can guarantee that $r(x_2) \in \mathbb{W}(x)$ holds for all $x \in \mathbb{X}$. 
The situation is displayed in Figure \ref{fig.nonlinearity}.

\begin{figure} \label{fig.nonlinearity} \centering
\includegraphics[width=0.7\textwidth]{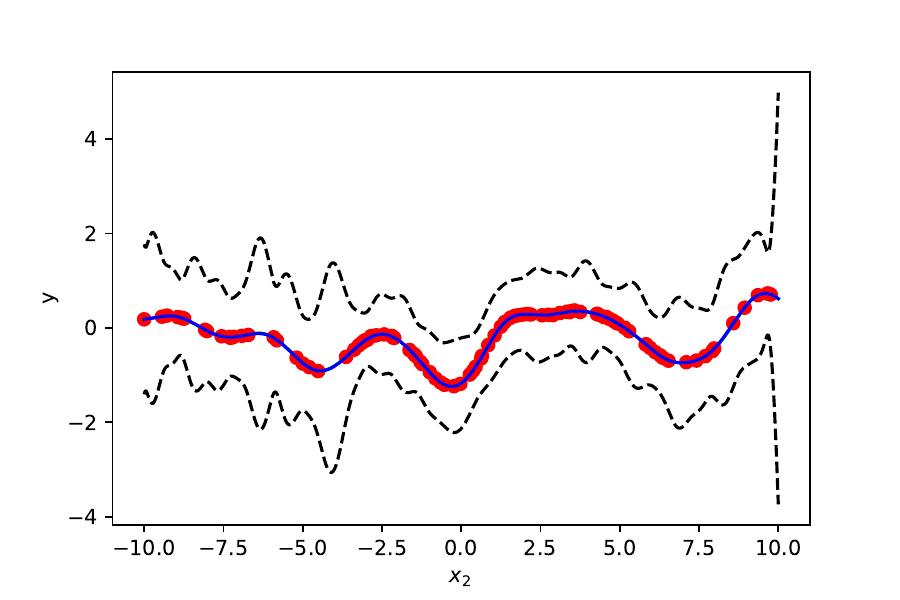}
\caption{Example nonlinearity. From the target function (blue solid line) 100 samples with noise (red dots) are sampled,
which are used to get the uncertainty sets (dashed black lines).}
\label{fig.nonlinearity}
\end{figure}

In order to follow \cite{solopertoetal_learning_rmpc_gp} as closely as possible, we exported the learning results and used the original Matlab
script to compute $\mathbb{Z}_k$ (provided by R. Soloperto). In order to save computation time, we decided to use a $50 \times 50$ state space grid and an MPC horizon of 9. 

The RMPC comes with deterministic guarantees. In particular, if the uncertainty is contained in the uncertainty sets, then
the RMPC controller ensures constrained satisfaction and convergence to a neighborhood of the origin, as well as a form of
Input-to-State stability, cf. Theorem 1 in \cite{solopertoetal_learning_rmpc_gp}. Since we can guarantee that the true uncertainty is covered
by the uncertainty sets with probability $1-\delta$ by Theorem 1, the deterministic guarantees hold with at least the same
probability.

\end{document}